\documentclass[twoside]{article}

\usepackage[accepted]{aistats2023}
\usepackage{algorithm, algorithmic, amsfonts, amsmath, amssymb, amsthm}
\usepackage{bbm, blindtext, booktabs}
\usepackage{color, commath}
\usepackage{diagbox}
\usepackage{graphicx}
\usepackage{hyperref}
\usepackage{mathtools, microtype, multicol, multirow}
\usepackage{physics, pifont}
\usepackage{scalerel, stackengine, subfigure}
\usepackage{tabularray}
\usepackage{xcolor}

\stackMath
\newcommand\reallywidehat[1]{%
\savestack{\tmpbox}{\stretchto{%
  \scaleto{%
    \scalerel*[\widthof{\ensuremath{#1}}]{\kern.1pt\mathchar"0362\kern.1pt}%
    {\rule{0ex}{\textheight}}%WIDTH-LIMITED CIRCUMFLEX
  }{\textheight}% 
}{2.4ex}}%
\stackon[-6.9pt]{#1}{\tmpbox}%
}
\newtheorem{definition}{Definition}
\newtheorem{theorem}{Theorem}
\newtheorem{lemma}{Lemma}
\newtheorem{proposition}{Proposition}

\usepackage[round]{natbib}

\begin{document}



\twocolumn[

\aistatstitle{Covariate-informed Representation Learning \\to Prevent Posterior Collapse of iVAE}

\aistatsauthor{ Young-geun Kim \And Ying Liu \And Xuexin Wei}

\aistatsaddress{Department of Psychiatry and\\ Department of Biostatistics\\ Columbia University \And  Department of Psychiatry and\\ Department of Biostatistics\\ Columbia University \And Department of Neuroscience and\\ Department of Psychology\\ The University of Texas at Austin} ]

\begin{abstract}
The recently proposed identifiable variational autoencoder (iVAE) framework provides a promising approach for learning latent independent components (ICs). iVAEs use auxiliary covariates to build an identifiable generation structure from covariates to ICs to observations, and the posterior network approximates ICs given observations and covariates. Though the identifiability is appealing, we show that iVAEs could have local minimum solution where observations and the approximated ICs are independent given covariates.
-- a phenomenon we referred to as the posterior collapse problem of iVAEs. To overcome this problem, we develop a new approach, covariate-informed iVAE (CI-iVAE) by considering a mixture of encoder and posterior distributions in the objective function. In doing so, the objective function prevents the posterior collapse, resulting latent representations that contain more information of the observations. Furthermore, CI-iVAEs extend the original iVAE objective function to a larger class and finds the optimal one among them, thus having tighter evidence lower bounds than the original iVAE. Experiments on simulation datasets, EMNIST, Fashion-MNIST, and a large-scale brain imaging dataset demonstrate the effectiveness of our new method.
\end{abstract}

\section{Introduction}\label{sec:introduction}

Representation learning aims to identify low-dimensional latent representation that can be used to infer the structure of data generating process, to cluster observations by semantic meaning, and to detect anomalous patterns \citep{bengio2009learning, bengio2013representation}.
Recent progress on computer vision has shown that deep neural networks are effective in learning representations from rich high-dimensional data \citep{lecun2015deep}.
Now there is emerging interest in utilizing deep neural networks in scientific exploratory analysis to learn the representation of high-dimensional genetic/brain imaging data associated with phenotypes \citep{huang2017modeling, pinaya2019using, han2019variational, qiang2020deep, kim2021representation, lu2021deep}.
In these scientific applications, autoencoder (AE,~\citealt{bengio2007greedy}) and variational autoencoder (VAE,~\citealt{kingma2014auto}) are popular representation learning methods. Compared to VAEs, the recently proposed identifiable VAE (iVAE,~\citealt{khemakhem2020variational}) provides appealing properties for scientific data analysis \citep{zhou2020learning, schneider2022learnable}: (i) the identifiability of the learned representation; (ii) the representations are assumed to be associated with observed covariates. For example, in human health researches, (i) and (ii) are essential to identify latent independent components and learn representations associated with gender, age, and ethnicity, respectively.
However, we find that when applying iVAEs to various datasets including a human brain imaging dataset, iVAEs sometimes converge to a bad local optimum where representations depend only on covariates (e.g., age, gender, and disease type), thus it is not a good representation for the observations (e.g., genetics, brain imaging data). Therefore, it is necessary to modify the objective function of iVAEs to learn better representations for scientific applications.

VAE consists of two networks: (i) the decoder network maps the prior latent independent components (ICs) to generate observations; (ii) the encoder network approximates the distribution of ICs given observations. iVAEs extend VAEs by introducing auxiliary covariates into the encoder and prior distributions, and construct identifiable data generation processes from covariates to ICs to observations. Inspired by the iVAE framework, recently, some identifiable generative models have been proposed. \cite{kong2022partial} and \cite{wang2022causal} founded identifiable generative models for domain-adaptation and causal inference, respectively. \cite{zhou2020learning} extended the data generation structure of iVAEs from continuous observation noise cases to Poisson. \cite{sorrenson2020disentanglement} explicitly imposed the inverse relation between encoders and decoders via general incompressible-flow networks (GIN), volume-preserving invertible neural networks.
In this work, we focus on a problem of the current iVAEs implementation that could converge to a bad local optimum that lead to uninformative representations, and we propose a new approach that solves this problem by modifying objective functions.

Though the identifiability is appealing, we observe that representations from iVAEs ignore observations in many cases with experiments. 
The main reason is the Kullback-Leibler (KL) term in the evidence lower bounds (ELBOs,~\citealt{bishop2006pattern}) enforcing the posterior distribution of iVAEs to be prior distributions.
This phenomenon is similar to the posterior collapse problem of VAEs where estimated ICs by encoders are independent of observations \citep{bowman2015generating, lucas2019don, he2019lagging, dai2020usual}. A detailed review on the posterior collapse problem is provided in Section \hyperref[sec2.2]{2.2}. We extend the notion of posterior collapse problem of VAEs to formulate this undesirable property of iVAEs, and coin it as the posterior
collapse problem of iVAEs. With the formulation, we theoretically derive that iVAEs occur this problem under some conditions.

To overcome the limitation of iVAEs, we have developed a new method, the Covariate-Informed Identifiable VAE (CI-iVAE). Our new method leverages encoders in addition to the original posterior distribution considered in the previous iVAE to derive a new family of objective functions (ELBOs) for model fitting.\footnote{We distinguish encoders $q_{\phi}(z|x)$ and posterior $q_{\phi}(z|x,u)$ to avoid confusion where $z$, $x$, $u$, and $\phi$ indicate representations, observations, covariates, and network parameters, respectively. The posterior networks in iVAEs are different from encoders in usual VAEs.}
Crucially, in doing so, our objective function prevents the posterior collapse by modifying the KL term. CI-iVAEs extend the iVAE objective function to a larger class and finds the samplewise optimal one among them.

We demonstrate that our method can more reliably learn features of various synthetic datasets, two benchmark image datasets, EMNIST \citep{cohen2017emnist} and Fashion-MNIST \citep{xiao2017fashion}, and a large-scale brain imaging dataset for adolescent mental health research. Especially, we apply our method and iVAEs to a brain imaging dataset, Adolescent Brain Cognitive Development (ABCD) study \citep{jernigan2018adolescent}.\footnote{The ABCD dataset can be found at \url{https://abcdstudy.org}, held in the NIMH Data Archive (NDA).}
Our real data analysis on the ABCD dataset is the first application of identifiable neural networks in human brain imaging. Our method successfully learns representations from brain imaging data associated with the characteristics of subjects while the iVAE will learn non-informative representation due to posterior collapse problem on this real-data application.

Our contributions can be summarized as follows:
\begin{itemize}
\item We formulate the posterior collapse problem of iVAEs and derive that iVAEs may learn collapsed posteriors.
\item We propose CI-iVAEs to learn better representations than iVAE by modifying the ELBO to prevent the posterior collapse.
\item Experiments demonstrate that our method out performed iVAEs by preventing the posterior collapse problem.
\item Our work is the first to learn ICs of human brain imaging with identifiable generative models.
\end{itemize}

All proofs of theoretical results are provided in Appendix \hyperref[app.A]{A}. Implementation details are provided in Appendix \hyperref[app.B]{B}.

\section{Related Prior Work}\label{sec2}

\subsection{Generative Autoencoders}\label{sec2.1}
Generative autoencoders are one of the prominent directions for representation learning \citep{higgins2016beta}. They usually describe data generation processes with joint distributions of latent ICs and observations \citep{kingma2014semi}, and optimize reconstruction error with penalty terms \citep{kingma2014auto}. The reconstruction error is a distance between observations and their reconstruction results by encoders and decoders. For example, ELBO, the objective function of VAEs is a summation of the reconstruction probability (negative reconstruction error) and the KL divergence between encoder and prior distributions.

In the cases where auxiliary covariates are available, many conditional generative models have been proposed to incorporate covariates in generators in addition to latent variables. In conditional VAEs \citep{sohn2015learning} and conditional adversarial AEs \citep{makhzani2015adversarial}, covariates are feed-forwarded by both encoder and decoder. Auxiliary classifiers for covariates are often applied to learn representations that can generate better results \citep{kameoka2018acvae}.
The aforementioned methods have shown prominent results, but their models can learn the distribution of observations with many different prior distributions of latent variables, i.e., they are not \textit{identifiable} \citep{khemakhem2020variational}.

\subsection{Posterior Collapse}\label{sec2.2}
The representations by VAEs are often poor due to the posterior collapse, which has been pointed as a practical drawback of VAEs and their variations \citep{kingma2016improved, yang2017improved, dieng2019avoiding}.
The posterior collapse refers to the phenomenon that the posterior converges to the prior. In this case, approximated ICs are independent of observations, i.e., representations lose the information in observations. One reason for this is that the KL divergence term in the ELBO enforces the encoders to be close to prior distributions \citep{huang2018improving, razavi2018preventing}.

A line of works has focused on posterior distributions and the KL term to alleviate this issue. \cite{he2019lagging} aggressively optimized the encoder whenever the decoder is updated, \cite{kim2018semi} introduced stochastic variational inference \citep{hoffman2013stochastic} to utilize instance-specific parameters for encoders, and \cite{fu2019cyclical} monotonically increased the coefficient of the KL term from zero to ensure that posteriors are not collapsed in early stages.
However, \cite{dai2020usual} derived that VAEs sometimes have lower values of ELBOs than posterior collapse cases. It means that the surface of ELBOs naturally results in a bad local optima, posterior collapse cases. We will extend this theory, and show that iVAEs may result in local optima with collapsed posteriors, in which case the representations are independent to observations given covariates in Section \hyperref[sec3.2]{3.2}.
Recently, \cite{wang2021posterior} prevented the posterior collapse of VAEs with the latent variable identifiability, which refers to that distributions identify \textit{latent variables} for given model parameters. It is important to note that latent variable identifiability is different from the model identifiability as discussed in the iVAE framework. The latter refers to that distributions identify \textit{model parameters}. For brevity, we will use the term identifiability to refer to the model identifiability in this paper.

\section{Proposed Method}\label{sec3}
\subsection{Preliminaries}
\subsubsection{Basic Notations and Assumptions}
We denote observations, covariates, and latent variables by $X \in \mathbb{R}^{d_{X}}$, $U \in \mathbb{R}^{d_{U}}$, and $Z \in \mathbb{R}^{d_{Z}}$, respectively. The dimension of latent variables is lower than that of observations, i.e., $d_{Z}<d_{X}$. For a given random variable (e.g., $Z$), its realization and probability density function (p.d.f.) is denoted by lower case (e.g., $z$) and $p$ (e.g., $p(z)$), respectively.
We distinguish encoder and posterior distributions and denote them by $q_{\phi}(z|x)$ and $q_{\phi}(z|x,u)$, respectively.

\subsubsection{Identifiable Variational Autoencoders}\label{sec3.1.3}
The identifiability is an essential property to recover the true data generation structure and to conduct correct inference \citep{lehmann2005testing, casella2021statistical}. A generative model is called identifiable if the distribution of generation results identifies parameters \citep{rothenberg1971identification, koller2009probabilistic}.
The iVAE framework provides an appealing approach for learning latent ICs. The iVAE assumes the following data generation structure: \begin{equation}\label{eq:iVAE_data_generation}
    \left\{
\begin{array}{l}
      Z|U \sim p_{T_{0}, \lambda_{0}}(z|u)\\
      X = f_{0}(Z) + \epsilon
      \end{array}
\right.
\end{equation}
where $Z$ denotes the IC (or \textit{source}) and $f_{0}$ denotes the nonlinear mixing function.
Here, $p_{T_{0}, \lambda_{0}}$ is a conditionally factorial exponential family distribution with sufficient statistics $T_{0}$ and natural parameters $\lambda_{0}$, and $\epsilon$ is an observation noise. The iVAE models the mixing function with neural networks, a flexible nonlinear model, and its generation process is identifiable under certain conditions. Key components of iVAEs include label prior, decoder, and posterior networks. The label prior and decoder, respectively, models the conditional distribution of latent variables given covariates, $p_{T_{0}, \lambda_{0}}(z|u)$, and observations given latent variables, $p_{f_{0}}(x|z)$. The posterior networks $q_{\phi}(z|x,u)$ approximates the posterior distribution of the ICs, $p_{f_{0}, T_{0}, \lambda_{0}}(z|x,u)$. Again, we distinguish encoders $q_{\phi}(z|x)$ and posteriors $q_{\phi}(z|x,u)$ to avoid confusion. The posterior networks in iVAEs approximate distributions of ICs given observations and covariates, which is different from encoders in usual VAEs.

An essential condition on the label prior to ensure the identifiability is the conditionally factorial exponential family distribution assumption. The label prior network is denoted by $p_{T, \lambda}(z|u)$ and can be expressed as $p_{T, \lambda}(z|u)=\prod_{i=1}^{d_{Z}}p_{T_{i}, \lambda_{i}}(z_{i}|u)$ where $p_{T_{i}, \lambda_{i}}(z_{i}|u)=\exp\big( \lambda_{i}(u) \cdot T_{i}(z_{i}) -A(u)+B(z_{i})\big)$ is the exponential family distribution with parameters $\lambda_{i}(u)$, sufficient statistics $T_{i}(z_{i})$, and known functions $A$ and $B$. We denote $T:=(T_{1},\dots,T_{d_{Z}})$ and $\lambda:=(\lambda_{1},\dots,\lambda_{d_{Z}})$. The decoder network is denoted by $p_{f}(x|z):=p(\epsilon=x-f(z))$ where $f$ is the modeled mixing function and $p(\epsilon)$ is the p.d.f. of noise variables $\mathcal{E}$. With label prior and decoder networks, the data generation process can be expressed as $p_{\theta}(x,z|u):= p_{f}(x|z)p_{T,\lambda}(z|u)$ where $\theta = (f, T, \lambda)$ is all the parameters for the data generation process.
The posterior network is denoted by $q_{\phi}(z|x, u)$. The encoder can be estimated by $q_{\phi}(z|x)=\int q_{\phi}(z|x,u)p(u|x)du$ or separately modeled. In the implementation, iVAEs model $p_{T,\lambda}(z|u)$ and $q_{\phi}(z|x,u)$ with Gaussian distributions.

\cite{khemakhem2020variational} defined the model identifiability of the generation process by $p_{\theta}(x,z|u)$.
\begin{definition}\label{def1} (Identifiability,~\citealt{khemakhem2020variational}) The $p_{\theta}(x,z|u)$ is called \textit{identifiable} if the following holds: for any $\theta=(f, T, \lambda)$ and $\tilde{\theta}=(\tilde{f}, \tilde{T}, \tilde{\lambda})$, $p_{\theta}(x|u)=p_{\tilde{\theta}}(x|u)$ implies $\theta \sim \tilde{\theta}$. Here, $\theta \sim \tilde{\theta}$ is defined as $T(f^{-1}(x))=\tilde{T}(\tilde{f}^{-1}(x))$ up to a invertible affine transformation.
\end{definition}
With the identifiability, finding the maximum likelihood estimators (MLEs) implies learning the true mixing function and ICs in \eqref{eq:iVAE_data_generation}.
\cite{khemakhem2020variational} showed that the identifiability holds and the affine transformation is component-wise one-to-one transformations as in ICA if $\lambda$ can make invertible matrix $(\lambda(u_{1})-\lambda(u_{0}), \dots, \lambda(u_{nk})-\lambda(u_{0}))$ with some $nk+1$ distinct $u_{0},\dots,u_{nk}$, and some mild conditions hold.

The objective function of iVAEs, ELBO is with respect to (w.r.t.) the conditional log-likelihood of observations given covariates $\log p_{\theta}(x|u)$. The ELBO can be expressed as
\begin{equation}\label{eq:iVAE_objective}
    \mathbb{E}_{q_{\phi}(z|x,u)}\log p_{f}(x|z) - \mathcal{D}_{\text{KL}}(q_{\phi}(z|x,u)||p_{T,\lambda}(z|u))
\end{equation}
which equals to $\log p_{\theta}(x|u) - \mathcal{D}_{\text{KL}}(q_{\phi}(z|x,u)||p_{\theta}(z|x,u))$. When the space of $q_{\phi}(z|x,u)$ includes $p_{\theta}(z|x,u)$ for any $\theta$, \eqref{eq:iVAE_objective} can approximate $\log p_{\theta}(x|u)$, which justifies that iVAEs learn the ground-truth data generation structure \citep{khemakhem2020variational}. However, we show in the next section that Equation (3) has a bad local optimum due to the KL term and iVAEs sometimes yield representations depending only on covariates by converging to this local optimum.

\subsection{Motivation}\label{sec3.2}
In this section, we describe how the objective function used in previous iVAEs produces bad local optimums for posterior $q_{\phi}(z|x,u)$. We further show how modifying the objective function with encoders $q_{\phi}(z|x)$ can alleviate this problem.

We first describe the posterior collapse problem of VAEs, and then formulate the bad local optimums of \eqref{eq:iVAE_objective}. In the usual VAEs using only observations, the objective function is $\mathbb{E}_{q_{\phi}(z|x)}\log p_{f}(x|z) - \mathcal{D}_{\text{KL}}(q_{\phi}(z|x)||p(z))$ where $p(z)$ is the prior distribution. The posterior collapse of VAEs can be expressed as $q_{\phi}(z|x)=p(z)$, and one reason of this phenomenon is the KL term enforcing $q_{\phi}(z|x)$ be close to $p(z)$.
Similarly, the KL term in \eqref{eq:iVAE_objective} enforces posterior distributions $q_{\phi}(z|x,u)$ to be close to label prior $p_{T,\lambda}(z|u)$.
We extend the notion of the posterior collapse of VAEs to formulate a bad local solution of \eqref{eq:iVAE_objective}, and coin it as the posterior collapse problem of iVAEs.
\begin{definition} (Posterior collapse of iVAEs) For a given dataset $\{(x_{i}, u_{i})\}_{i=1}^{n}$, we call the posterior $q_{\phi}(z|x,u)$ in iVAEs is collapsed if $q_{\phi}(z|x_{i}, u_{i})=p_{T, \lambda}(z|u_{i})$ holds for all $i=1,\dots,n$.
\end{definition}
In the following, we use the term posterior collapse to refer the posterior collapse of iVAEs. Under the posterior collapse, approximated ICs are independent of observations given covariates, i.e., we lose all the information in observations independent of covariates.
Furthermore, we derive that the posterior collapse is a local optimum of \eqref{eq:iVAE_objective} under some conditions, which is consistent with that the posterior collapse problem of VAEs is a local optimum of the objective function of VAEs \citep{dai2020usual}. In the following theorem, we assume two conditions formulated by \cite{dai2020usual}: $(C1)$ the derivative of reconstruction error is Lipschitz continuous and $(C2)$ the reconstruction error is an increasing function w.r.t. the uncertainty of latent variables. A detailed formulation of $(C1)$ and $(C2)$ is provided in Appendix A.
\begin{theorem}
Let $D=\{(x_{i}, u_{i})\}_{i=1}^{n}$ be samples from \eqref{eq:iVAE_data_generation} when $\epsilon \sim N(0, \gamma I)$ and the loss be the negative expectation of \eqref{eq:iVAE_objective} over $D$. For any iVAEs satisfying $(C1)$ and $(C2)$, there is a posterior collapse case whose loss value is lower than that of the iVAEs when $\gamma$ is sufficiently large.
\end{theorem}
That is, the objective function of existing iVAEs evaluates posterior collapse cases as better solutions than iVAEs. Roughly speaking, when the noise of the observation, $\gamma$, is large, the KL term in \eqref{eq:iVAE_objective} dominates the first term.

Our method modifies ELBOs to alleviate the posterior collapse problem of iVAEs. We note that the KL term $\mathcal{D}_{\text{KL}}(q_{\phi}(z|x,u)||p_{T, \lambda}(z|u))$ is a main reason of the posterior collapse and change $q_{\phi}(z|x,u)$ to linear mixtures of posterior $q_{\phi}(z|x,u)$ and encoder $q_{\phi}(z|x)$. For a motivating example, we can consider an alternative objective function (which can be shown to be an ELBO): 
\begin{equation}\label{eq:VAE_with_label_prior_objective}
    \mathbb{E}_{q_{\phi}(z|x)}\log p_{f}(x|z) - \mathcal{D}_{\text{KL}}(q_{\phi}(z|x)||p_{T, \lambda}(z|u)).
\end{equation}
It is equivalent to $\log p_{\theta}(x|u) - \mathcal{D}_{\text{KL}}(q_{\phi}(z|x)||p_{\theta}(z|x,u))$. It can be viewed as using encoder $q_{\phi}(z|x)$ to approximate the posterior of label prior distributions $p_{\theta}(z|x,u)$.
We name \eqref{eq:iVAE_objective} and \eqref{eq:VAE_with_label_prior_objective} by ELBOs with $q_{\phi}(z|x,u)$ and with $q_{\phi}(z|x)$, respectively.
We derive that \eqref{eq:VAE_with_label_prior_objective} prevents the posterior collapse problem of iVAEs. We say a label prior $p_{T,\lambda}$ non-trivial if $p_{T,\lambda}(z|u) \neq \int p_{T,\lambda}(z|u)p(u)du$ holds with positive probability w.r.t. $p(u)$, i.e., the label prior does not ignore covariates. All the label prior of iVAEs are non-trivial since $\lambda$ should make invertible matrix $(\lambda(u_{1})-\lambda(u_{0}),\dots,\lambda(u_{nk})-\lambda(u_{0}))$ with some $nk+1$ distinct $u_{0},\dots,u_{nk}$.
\begin{proposition}\label{prop:prevent_posterior_collapse}
For any non-trivial label prior $p_{T,\lambda}$, $q_{\phi}(z|x) \neq p_{T,\lambda}(z|u)$ holds with positive probability w.r.t. $p(x, u)$.
\end{proposition}
That is, the $q_{\phi}(z|x)$ can not collapse to non-trivial $p_{T,\lambda}(z|u)$ since it uses only observations. However, we derive that \eqref{eq:VAE_with_label_prior_objective} can not approximate $\log p_{\theta}(x|u)$.
\begin{proposition}\label{prop2} We assume that, for any $\theta$, there is $\phi$ satisfying $q_{\phi}(z|x)=p_{\theta}(z|x)$ with probability $1$ w.r.t. $p(x)$. For any $\theta$ forming non-trivial label prior,
\begin{equation*}\resizebox{1.0\hsize}{!}{
    $\begin{aligned}
        &\underset{\phi}{\max}\mathbb{E}_{p(x,u)}\big( \mathbb{E}_{q_{\phi}(z|x)}\log p_{f}(x|z) - \mathcal{D}_{\text{KL}}(q_{\phi}(z|x)||p_{T, \lambda}(z|u)) \big)\\
        &<\mathbb{E}_{p(x,u)}\log p_{\theta}(x|u).
    \end{aligned}$}
\end{equation*}
\end{proposition}
Thus, although the ELBO with $q_{\phi}(z|x)$ prevents the posterior collapse problem, it is strictly smaller than $\log p_{\theta}(x|u)$ even with the global optimum of $\phi$. In contrast, the ELBO with $q_{\phi}(z|x,u)$ can approximate the log-likelihood if we find the global optimum based on the data, but it may converge to the posterior collapse cases. Thus, in practice, it is desirable to find a good balance between these two considerations. In the next section, we develop a new method by modifying the ELBOs to achieve such a goal.

\subsection{Covariate-informed Identifiable VAE}
\begin{algorithm}[t]
\label{alg:training}
\caption{Training CI-iVAEs}
\begin{algorithmic}
\STATE \textbf{Input}: Training samples $\{(x_{i}, u_{i})\}_{i=1}^{n}$ and batch size $B$
\STATE \textbf{Output}: CI-iVAEs with ($f^{*}$, $T^{*}$, $\lambda^{*}$, $\phi^{*}$).
\end{algorithmic}
\begin{algorithmic}[1]
\STATE Initialize $(f, T, \lambda, \phi)$
\STATE \textbf{While} $(f, T, \lambda, \phi)$ did not converge \textbf{Do} \\
\STATE \quad Sample $\{(x_{i(b)}, u_{i(b)})\}_{b=1}^{B}$ from training samples
\STATE \quad Calculate samplewise optimal $\alpha^{*}(x_{i(b)}, u_{i(b)})$
\STATE \quad Update $(f, T, \lambda, \phi)$ by ascending $$B^{-1}\sum_{b=1}^{B}\text{ELBO}_{\theta,\phi}(\alpha^{*}(x_{i(b)}, u_{i(b)});x_{i(b)},u_{i(b)})$$
\STATE $(f^{*}, T^{*}, \lambda^{*}, \phi^{*}) \leftarrow (f, T, \lambda, \phi)$
\end{algorithmic}
\end{algorithm}
In this section, we provide our method, CI-iVAE. We use the same network architecture of iVAEs described in Section \hyperref[sec3.1.2]{3.1.2} to inherit the identifiability of the likelihood model. Our key innovation is the development of a new class of objective functions (ELBOs) to prevent the posterior collapse.

We consider mixtures of distributions by encoders $q_{\phi}(z|x)$ and the posterior in the original iVAE, $q_{\phi}(z|x, u)$,
\begin{equation}\label{eq:mixture_post}
    \{\alpha(x,u) q_{\phi}(z|x)+(1-\alpha(x,u))q_{\phi}(z|x,u)| \alpha(x,u) \in [0, 1]\},
\end{equation}
to derive ELBOs avoiding posterior collapse while approximating log-likelihoods. For simplicity, we use $\alpha$ to refer $\alpha(x,u)$, when there is no confusion. Any element in \eqref{eq:mixture_post} can provide a lower bound of log-likelihood. We first formulate a set of ELBOs using \eqref{eq:mixture_post} and then provide our method.
\begin{proposition}\label{prop:ELBO} For any sample $(x, u)$, $\theta=(f, T, \lambda)$, $\phi$, and $\alpha \in [0, 1]$, ELBO$_{\theta,\phi}(\alpha;x, u)$ defined as
\begin{equation*}
\begin{split}
    &\mathbb{E}_{\alpha q_{\phi}(z|x)+(1-\alpha)q_{\phi}(z|x,u)} \log p_{f}(x|z)\\
    &-\mathcal{D}_{\text{KL}}(\alpha q_{\phi}(z|x) + (1-\alpha)q_{\phi}(z|x,u)||p_{T,\lambda}(z|u))
\end{split}
\end{equation*}
is a lower bound of $\log p_{\theta}(x|u)$.
\end{proposition}
Proposition \hyperref[prop:ELBO]{3} can be derived by performing variational inference with $\alpha(x,u)q_{\phi}(z|x) + (1-\alpha(x,u))q_{\phi}(z|x,u)$. A set of ELBOs, $\{ \text{ELBO}_{\theta, \phi}(\alpha;x, u)| \alpha\in [0, 1]\}$ is a continuum of ELBOs whose endpoints are ELBOs with $q_{\phi}(z|x,u)$ and with $q_{\phi}(z|x)$.

We next present our CI-iVAE method. For a given identifiable generative model, CI-iVAE uses covariates to find the samplewise optimal elements $\alpha^{*}(x, u):=\underset{\alpha  \in [0, 1]}{\arg\max}\text{ELBO}_{\theta, \phi}(\alpha;x, u)$ in \eqref{eq:mixture_post} and utilizes it to maximize the tightest ELBOs. We refer to $\alpha^{*}(x,u)q_{\phi}(z|x)+(1-\alpha^{*}(x,u))q_{\phi}(z|x,u)$ as the samplewise optimal posterior distributions. The objective function of the CI-iVAE method is given by
\begin{equation}\label{eq:CI-iVAE_objective}
    \text{ELBO}_{\theta, \phi}(\alpha^{*}(x,u); x, u),
\end{equation}
where the KL term in \eqref{eq:iVAE_objective} is changed to $\mathcal{D}_{\text{KL}}(\alpha^{*}(x,u)q_{\phi}(z|x)+(1-\alpha^{*}(x,u))q_{\phi}(z|x,u)||p_{T,\lambda}(z|u))$ whose bad local solutions are $\alpha^{*}(x,u)q_{\phi}(z|x)+(1-\alpha^{*}(x,u))q_{\phi}(z|x,u)=p_{T,\lambda}(z|u)$. We derive that the posterior collapse problem does not occur at this local solution.
\begin{theorem}\label{thm:prevent_posterior_collapse}
For any $\theta$ forming non-trivial label prior and $\phi$, if $\alpha^{*}(x,u)>0$ and $\alpha^{*}(x,u) q_{\phi}(z|x) + (1-\alpha^{*}(x,u))q_{\phi}(z|x,u)=p_{T,\lambda}(z|u)$, then $q_{\phi}(z|x,u) \neq p_{T,\lambda}(z|u)$ holds with positive probability.
\end{theorem}
That is, even in the worst case where the KL term dominates \eqref{eq:CI-iVAE_objective}, the $q_{\phi}(z|x,u)$ does not collapse to $p_{T,\lambda}(z|u)$.

Furthermore, we derive that our ELBO can approximate the log-likelihood.
\begin{proposition}\label{prop4} We assume that, for any $\theta$, there is $\phi$ satisfying $q_{\phi}(z|x,u)=p_{\theta}(z|x,u)$ with probability  $1$ w.r.t. $p(x, u)$.\footnote{Again, \cite{khemakhem2020variational} assumed this condition to justify iVAEs using $q_{\phi}(z|x,u)$.} Then, $\underset{\phi}{\max}\mathbb{E}_{p(x,u)}\text{ELBO}_{\theta, \phi}(\alpha^{*}(x,u);x,u)=\mathbb{E}_{p(x,u)}\log p_{\theta}(x|u)$.
\end{proposition}
Proposition \hyperref[prop4]{4} implies that maximizing our ELBO is equivalent to learning MLEs with tighter lower bounds than \eqref{eq:iVAE_objective}. We also derive that the difference between our optimal ELBO and the ELBO of existing iVAE is significant under some conditions in Theorem 3 in Appendix A.

Though the CI-iVAE allows learning better representations, numerical approximation for $\alpha^{*}(x,u)$ in \eqref{eq:CI-iVAE_objective} may lead to burdensome computational costs. To overcome this technical obstacle, we provide an alternative expression of our ELBO that allows us to approximate $\alpha^{*}(x,u)$ within twice the computation time of the ELBO of the existing iVAE.
\begin{proposition}\label{prop:ELBO_representation} For any sample $(x, u)$, $\theta=(f, T, \lambda)$, $\phi$, and $\alpha \in [0, 1]$, $\text{ELBO}_{\theta, \phi}(\alpha;x,u)$ can be expressed as
\begin{equation*}\resizebox{1.0\hsize}{!}{
    $\begin{aligned}
    &\alpha \text{ELBO}_{\theta, \phi}(1;x,u)+(1-\alpha)\text{ELBO}_{\theta, \phi}(0;x,u)\\
    &+\alpha \mathcal{D}^{1-\alpha}_{\text{Skew}}(q_{\phi}(z|x)||q_{\phi}(z|x,u))+(1-\alpha) \mathcal{D}^{\alpha}_{\text{Skew}}(q_{\phi}(z|x,u)||q_{\phi}(z|x)),
\end{aligned}$}
\end{equation*}
where $\mathcal{D}^{\alpha}_{\text{Skew}}$ is the skew divergence defined as $\mathcal{D}^{\alpha}_{\text{Skew}}(p||q):=\mathcal{D}_{\text{KL}}(p||(1-\alpha)p+\alpha q)$ \citep{lin1991divergence}.
\end{proposition}

Details on using samplewise optimal posterior distributions to maximize \eqref{eq:CI-iVAE_objective} is described in Algorithm \hyperref[alg:training]{1}. The main bottleneck is on computing $\text{ELBO}_{\theta, \phi}(0;x,u)$, $\text{ELBO}_{\theta, \phi}(1;x,u)$, and conditional means and standard deviations of the label prior and encoder distributions, which requires roughly twice the computation time of the ELBO of iVAE. With Proposition \hyperref[prop:ELBO_representation]{5} and reparametrization trick, for any $\alpha$, the remaining process to compute $\text{ELBO}_{\theta, \phi}(\alpha;x,u)$ is completed in a short time.

\section{Experiments}\label{sec4}
\begin{figure*}[h]
%\vspace{.3in}
\centerline{\includegraphics[width=1.8\columnwidth]{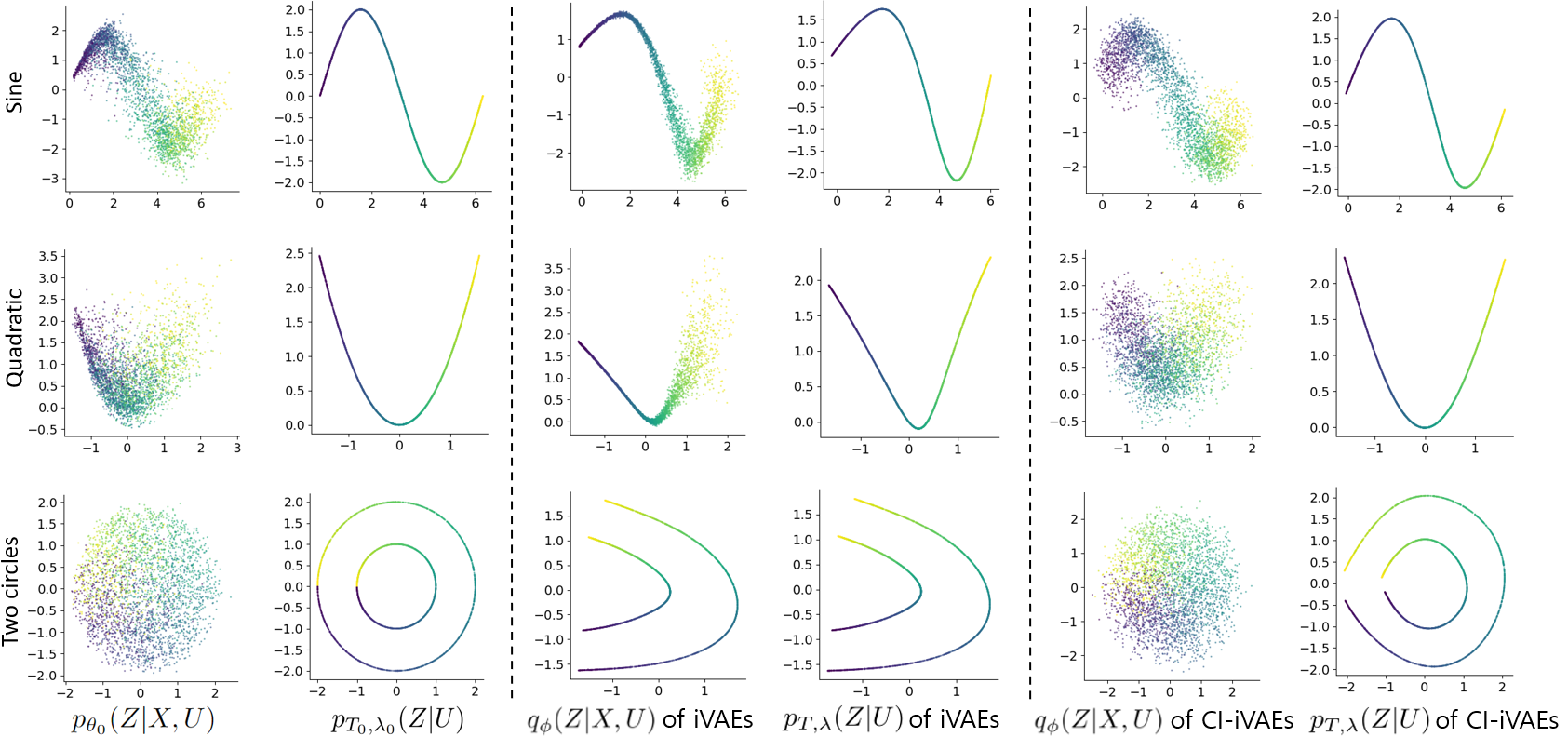}}
\vspace{.1in}
\caption{Visualization of latent variables from simulation datasets. In all datasets, the posterior $q_{\phi}(z|x,u)$ of iVAEs are close to the label prior $p_{T, \lambda}(z|u)$, and tend to underestimate the variability by observations. In contrast, by addressing samplewise optimal posteriors, CI-iVAEs learn posterior and prior distributions closer to GT.}
\label{fig:simulation_results}
\end{figure*}

\begin{figure}[h]
\centerline{\includegraphics[width=1.0\columnwidth]{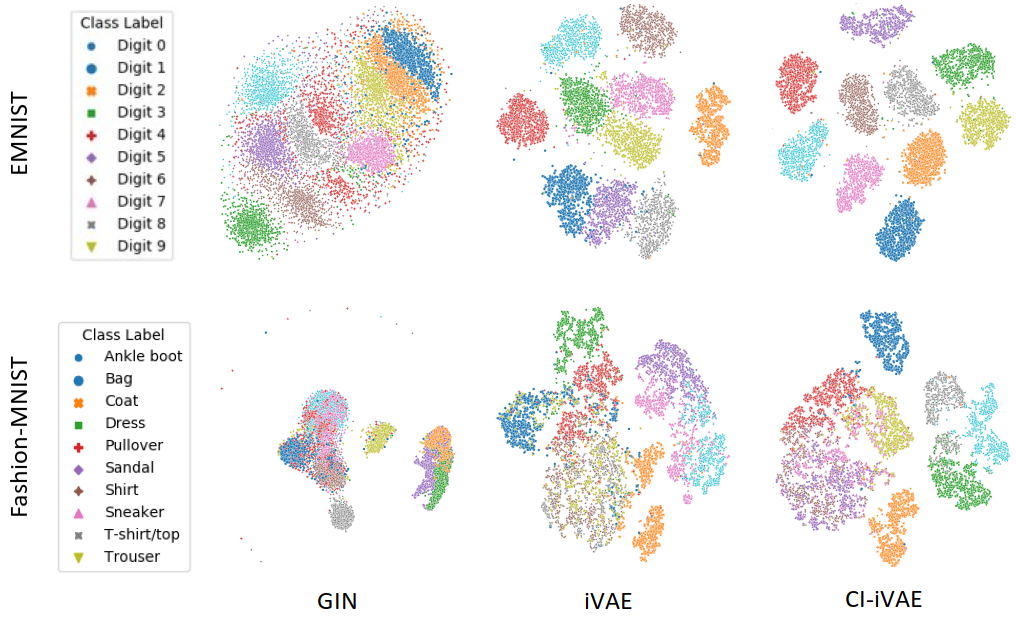}}
\vspace{.15in}
\caption{Visualization of the t-SNE embeddings of encoders $q_{\phi}(z|x)$ from various methods on EMNIST and Fashion-MNIST datasets.}
\label{fig:image_datasets_t-SNE}
\end{figure}

\begin{figure}[h]
\centerline{\includegraphics[width=1.0\columnwidth]{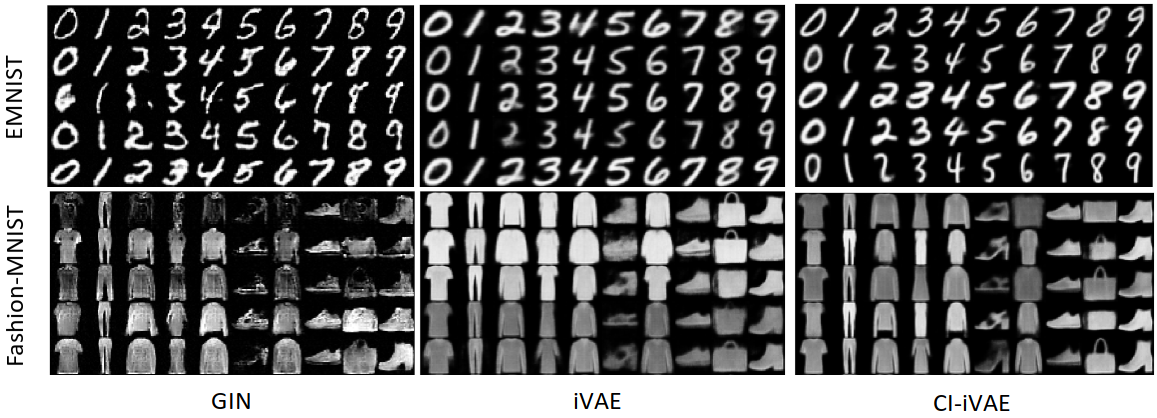}}
\vspace{.1in}
\caption{Generation results from various methods on EMNIST and Fashion-MNIST datasets. We generate five synthetic images in each class. Both iVAEs and CI-iVAEs produce clearer images than GIN.}
\label{fig:image_datasets_generation_result}
\end{figure}

\begin{figure}[h]
\centerline{\includegraphics[width=0.7\columnwidth]{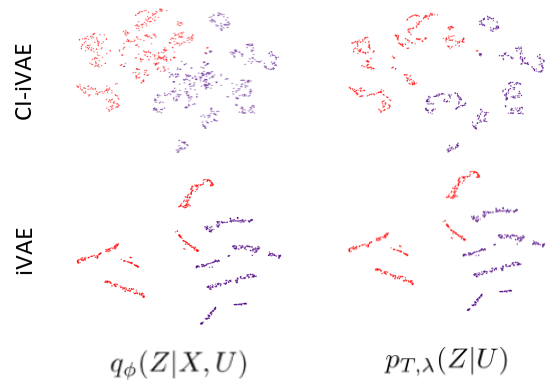}}
\vspace{.1in}
\caption{Visualization of the t-SNE embeddings of latent variables from iVAEs and CI-iVAEs on the ABCD dataset. Red and blue points indicate female and male, respectively. The posterior of iVAEs collapses to the label prior while that of CI-iVAEs learn variations by observations.}
\label{fig:ABCD_t_sne_representations}
\end{figure}
We have validated and applied our method on synthetic, EMNIST, Fashion-MNIST, and ABCD datasets. As in \cite{zhou2020learning}, we model the posterior distribution by $q_{\phi}(z|x,u) \propto q_{\phi}(z|x)p_{T,\lambda}(z|u)$. The label prior $p_{T,\lambda}(z|u)$ and encoder $q_{\phi}(z|x)$ is modeled as Gaussian distributions.
Implementation details, including data descriptions and network architectures, are provided in Appendix \hyperref[app.B.1]{B.1}.\footnote{The implementation code is provided in the supplementary file.}

\subsection{Simulation Study}\label{sec4.1}
We first examine the effectiveness of samplewise optimal posteriors by comparing iVAEs and CI-iVAEs on synthetic datasets.
We use three data generation schemes and named them by shapes of distributions of latent variables given covariates:  (i) sine, (ii) quadratic, and (iii) two circles.

\begin{table*}[h]\label{tab:simulation_performance}
\caption{Means of evaluation metrics with standard errors from iVAEs and CI-iVAEs on various latent structures. For all metrics, higher values are better. CI-iVAEs outperform iVAEs in all datasets and all metrics. The number of repeats is $20$.} \label{sample-table}
\begin{center}
\resizebox{2.0\columnwidth}{!}{\begin{tabular}{ccccccc} \toprule
Latent Structure & Method & \multicolumn{5}{c}{Evaluation Metric} \\ \cline{3-7} & & COD ($q_{\phi}(z|x,u)$) & COD ($q_{\phi}(z|x)$)  & MCC ($q_{\phi}(z|x,u)$) & MCC ($q_{\phi}(z|x)$) & Log-likelihood                                               \\ \hline \hline
\multirow{2}{*}{Sine} & iVAE  & .9285 (.0009) & .9692 (.0004) & .7364 (.0286) & .7531 (.0266) & -131.8387 (1.7467)\\
& CI-iVAE & \textbf{.9823} (.0002) & \textbf{.9782} (.0003) & \textbf{.8898} (.0114) & \textbf{.8716} (.0117) & \textbf{-111.2823} (0.5576) \\ \hline
\multirow{2}{*}{Quadratic} & iVAE & .5238 (.0059) & .8256 (.0112) & .5467 (.0086) & .6966 (.0097) & -105.4009 (0.1260)\\
& CI-iVAE & \textbf{.9177} (.0006) & \textbf{.8755} (.0013) & \textbf{.8463} (.0192) & \textbf{.8424} (.0192) & \textbf{-101.0430} (0.2420)\\ \hline
\multirow{2}{*}{Two Circles} & iVAE & .2835 (.0119) & .7753 (.0186) & .6233 (.0108) & .8171 (.0106) & -114.9876 (1.2964) \\
& CI-iVAE &  \textbf{.9156} (.0007) & \textbf{.9440} (.0008) & \textbf{.8278} (.0194)& \textbf{.8347} (.0196) & \textbf{-102.9267} (0.3198) \\ \bottomrule
\end{tabular}}
\end{center}
\end{table*}
Table \hyperref[tab:simulation_performance]{1} and Figure \hyperref[fig:simulation_results]{1} show the results from the synthetic datasets. We consider coefficient of determination (COD,~\citealt{schneider2022learnable}), mean correlation coefficient (MCC,~\citealt{khemakhem2020variational}), and log-likelihood as evaluation metrics. Higher values of CODs and MCCs indicate closer to the GT. The log-likelihood is used to evaluate the whole data generation scheme. All methods use the same network architectures to use the same family of density functions. The log-likelihood is calculated by $\log p_{\theta}(x|u)=\log \int p_{f}(x|z)p_{T,\lambda}(z|u)dz$ with Monte Carlo approximation and is averaged over samples.

According to Table \hyperref[tab:simulation_performance]{1}, iVAEs are worse than CI-iVAEs in all datasets and all evaluation metrics. When the iVAE is compared with the CI-iVAE, most of p-values for all datasets and all evaluation metrics are less $.001$. The only exception is the MCCs on $q_{\phi}(z|x)$ on the two circles dataset whose p-value is $.22$. That is, our sharper ELBO with samplewise optimal posteriors enhances performances on learning GT ICs and MLEs for mixing functions. We also observe that encoder distributions of CI-iVAEs yield performances comparable to posterior distributions.

Visualization of latent variables and results is presented in Figure \hyperref[fig:simulation_results]{1}. In the left panel, the GT posterior and prior distributions are presented. In the middle and right panels, results from iVAEs and CI-iVAEs are presented, respectively. Points indicate conditional expectations of each distributions and are colored by covariates. For the two circles structure, we color datapoints by angles, $U_{1}$.
In all latent structures, iVAEs learn $q_{\phi}(z|x,u)$ closer to $p_{T,\lambda}(z|u)$. The GT posterior distribution of ICs has variations by observations, and the $q_{\phi}(z|x,u)$ of iVAEs tends to underestimate these variations. In contrast, CI-iVAEs alleviate this phenomenon, consequently better recovering the GT posterior and prior distributions than iVAEs. Especially in the two circles dataset, we sample the angle $U_{1}$ from $[-\pi, \pi]$ to check whether each model can recover the connection at $U_{1}=-\pi$ and $U_{1}=\pi$. The points at $U_{1}=-\pi$ and $\pi$ by CI-iVAEs locate closer to each other than those by iVAEs.

\subsection{Applications to Real Datasets}
\subsubsection{EMNIST and Fashion-MNIST}
We compare the proposed CI-iVAE with GIN and iVAE on two benchmark image datasets: EMNIST and Fashion-MNIST. GIN is a state-of-the-art identifiable deep image generative model using convolutional coupling blocks \citep{sorrenson2020disentanglement}. GIN does not incorporate covariates, so we compare representations from encoders $q_{\phi}(z|x)$ from various methods. In EMNIST and Fashion-MNIST, we use labels of digits and fashion-items as covariates.
\begin{table}[t]
\caption{Means of evaluation metrics with standard errors from various methods on EMNIST and Fashion-MNIST datasets. The number of repeats is $20$.} \label{tab:image_performances}
\begin{center}
\begin{tabular}{ccc} \toprule
Dataset & Method & SSW/SST ($\downarrow$)                                                 \\ \hline \hline
\multirow{3}{*}{EMNIST} & GIN & .6130 (.0075) \\
& iVAE  & .5486 (.0037) \\
& CI-iVAE & \textbf{.4117} (.0032) \\ \hline
\multirow{3}{*}{Fashion-MNIST} & GIN & .8503 (.0026) \\
& iVAE & .6157 (.0046) \\
& CI-iVAE & \textbf{.4926} (.0024) \\ \bottomrule
\end{tabular}
\end{center}
\end{table}

\begin{table*}[h]
\caption{Means of prediction performances with standard errors from $5$-fold cross-validation. We train support vector regression for age and CBCL scores, and support vector machine for sex and puberty. The representation from CI-iVAEs outperform that from AEs and iVAEs baselines. Prediction performances with our representations are comparable to or better than those with raw observations.} \label{tab:abcd_performances}
\begin{center}
    \begin{tabular}{lrrrr}
\toprule Input & \multicolumn{4}{c}{Covariates} \\ \cline{2-5}
           & Age (MSE $\downarrow$)           & Sex (Error rate $\downarrow$)    & Puberty (F1 score $\uparrow$)  & CBCL scores (MSE $\downarrow$) \\ \hline \hline
    x            & .165 (.012)          & .105 (.007)          & .531 (.023)          & \textbf{.072} (.006) \\
    Representations from AE      & .336 (.005)          & .368 (.006)          & .188 (.009)          & .135 (.006) \\
    Representations from iVAE    & .361 (.005)          & .479 (.010)          & .046 (.005)          & .148 (.006) \\
    Representations from CI-iVAE & \textbf{.153} (.017) & \textbf{.086} (.009) & \textbf{.563} (.053) & .073 (.008)\\ \bottomrule
    \end{tabular}
\end{center}
\end{table*}

The experimental results are presented in Table \hyperref[tab:image_performances]{2} and Figure \hyperref[fig:image_datasets_t-SNE]{2}. Further generation results are provided in Figures 1 and 2 in Appendix \hyperref[app.B.2]{B.2}. We consider the ratio of the within-cluster sum of squares (SSW) over the total sum of squares (SST) as evaluation metrics. The SSW/SST measures how well representations are clustered by covariates.

Results presented in Table \hyperref[tab:image_performances]{2} show that our method is better than both iVAEs and GIN in all datasets. When the iVAE is compared with the CI-iVAE, p-values are $2.3 \times 10^{-27}$ for EMNIST and $8.8 \times 10^{-25}$ for Fashion-MNIST datasets. When GIN is compared with the CI-iVAE, p-values are $2.1 \times 10^{-25}$ for EMNIST and $3.1 \times 10^{-48}$ for Fashion-MNIST datasets. That is, the CI-iVAE is better than iVAE and GIN for extracting covariates-related information from observations.

A visualization of latent variables using t-SNE \citep{van2008visualizing} embeddings is presented in Figure \hyperref[fig:image_datasets_t-SNE]{2}. In all datasets, CI-iVAEs yield more separable representations of the covariates than GIN and iVAEs.

We present generation results in Figure \hyperref[fig:image_datasets_generation_result]{3}. iVAEs and CI-iVAEs tend to generate more plausible images than GIN. In the third row in EMNIST, GIN fails to produce some digits. When we compare iVAEs and CI-iVAEs, iVAEs tend to generate more blurry results. In the 6th column in Fashion-MNIST, iVAEs fail to generate the sandal image in the second row and generate blurry results in other rows.

\subsubsection{Application to Brain Imaging Data}
Here we present the application of the CI-iVAE on the ABCD study dataset, which is the largest single-cohort prospective longitudinal study of neurodevelopment and children's mental health in the United States.
The ABCD dataset provides resources to address a central scientific question in Psychiatry: can we find the brain imaging representations that are associated with phenotypes. For this purpose, the proposed CI-iVAE is appealing and more efficient in that it incorporates the information in demographics, symptoms, which, by domain knowledge, are associated with brain imaging. In this application, observations are the MRI mean thickness and functional connectivity data from Gorden Atlas \citep{gordon2016generation}, and the covariates are interview age, gender, puberty level, and total Child Behavior Checklist (CBCL) scores. To find the brain imaging representations that can be interpreted with covariates, we train iVAEs and evaluate representations from encoders $q_{\phi}(z|x)$ using only test brain imaging. With $q_{\phi}(z|x)$, we can extract information contained only in brain imaging. We consider vanilla AEs and iVAEs as baselines.

The experimental results are presented in Figure \hyperref[fig:ABCD_t_sne_representations]{4} and Table \hyperref[tab:abcd_performances]{3}. We quantify prediction performance using representations from various methods to evaluate how much covariate-related information in observations is extracted from brain measures. We use conditional expectations of encoders, $q_{\phi}(z|x)$, as representations. For the puberty level, we oversample minority classes to balance classes.

Visualization of t-SNE embeddings of latent variables is presented in Figure \hyperref[fig:ABCD_t_sne_representations]{4}.
The result from iVAEs demonstrates that $q_{\phi}(z|x,u)$ collapses to $p_{T, \lambda}(z|u)$ and iVAEs tend to learn less diverse representations than CI-iVAEs, which is consistent to previous results.

Table \hyperref[tab:abcd_performances]{3} shows that, for all covariates, representations from CI-iVAEs outperform those from AEs and iVAEs and are comparable to raw observations. The CI-iVAE is significantly better than the iVAE, p-values for age, sex, puberty level, and CBCL scores are $1.1\times10^{-6}$, $8.4\times10^{-10}$, $4.6\times10^{-6}$, and $3.1\times10^{-5}$, respectively. Due to the collapsed posterior, representations from encoders in iVAEs do not extract much information from brain imaging. In contrast, the encoder of CI-iVAEs extracts representations which are very informative of the covariates. Thus, our representations extracted from brain measures preserve more covariates-related information than the other methods.

\section{Discussion}\label{sec5}
We proposed a new representation learning approach, CI-iVAEs, to overcome the limitations of iVAEs. Our objective function uses samplewise optimal posterior distributions to prevent the posterior collapse problem.
Representations from our methods on various synthetic and real datasets were better than those from existing methods by extracting covariates-associated information in observations. Our work is the first to adapt identifiable generative models to human brain imaging. We used pre-processed ROI level summary measures as observations. An interesting future direction is to extract interpretable features from minimally-processed images. Another direction is to apply/develop interpretable machine learning tools creating feature importance scores \citep{ribeiro2019local, molnar2020interpretable, guidotti2018survey} to reveal scientific insights from our representations.

\section{Acknowledgement}
This work is supported by NIMH grant R01 MH124106.

\newpage
\appendix
\section{Details on Theoretical Results}\label{app.A}
\subsection{Further Theoretical Results}
Our method introduces both $q_{\phi}(z|x)$ and $q_{\phi}(z|x,u)$ in consisting posterior distributions to conduct variational inference. We derive that the proposed ELBOs are concave, are sharper lower bounds than the ELBO of the existing iVAE, and prevent the posterior collapse issue in the existing iVAE.\\ \\
We first show the concavity of the proposed ELBO and a necessary and sufficient condition for the concavity.
\label{prop:concavity}
\textbf{Proposition 6.} \textit{For any sample $(x, u)$ and $\alpha \in [0, 1]$, $\text{ELBO}_{\theta,\phi}(\alpha;x,u)$ is concave w.r.t. $\alpha$. It is strictly concave if and only if $q_{\phi}(z|x) \neq q_{\phi}(z|x,u)$ for some $z$.}

Note that the first two terms in Proposition 5 are linear w.r.t. $\alpha$, so the concavity comes from the last two terms. That is, the difference between $q_{\phi}(z|x)$ and $q_{\phi}(z|x,u)$ induces the concavity.

Next, we show that our method uses strictly sharper lower bounds than the existing iVAE. Let $\Delta_{1-0}(x,u):= \text{ELBO}_{\theta,\phi}(1;x,u)-\text{ELBO}_{\theta,\phi}(0;x,u)$ and
\begin{equation}\label{eq:snr}
    \text{SNR}(g):= \frac{(\mathbb{E}_{q_{\phi}(z|x)}g(z)-\mathbb{E}_{q_{\phi}(z|x, u)}g(z))^{2}}{\max(Var_{q_{\phi}(z|x)}g(z), Var_{q_{\phi}(z|x, u)}g(z))}.
\end{equation}
Here, $\text{SNR}(g)$ is the signal-to-noise ratio between $q_{\phi}(z|x)$ and $q_{\phi}(z|x,u)$ w.r.t. $g$ quantifying the discrepancy between the two distributions. The proof of Theorem \hyperref[thm:positive_margin]{3} is provided in Appendix \hyperref[AppendixA.3]{A.3}.\\
\textbf{Theorem 3.}\label{thm:positive_margin} \textit{For any sample $(x, u)$, $\theta=(f, T, \lambda)$, $\phi$, and $\epsilon>0$, if there is a function $g: \mathcal{Z} \to \mathbb{R}$ satisfying $\text{SNR}(g) \geq 1/\epsilon$ and $|\Delta_{1-0}(x,u)| \leq -\log\epsilon+O(\epsilon \log\epsilon)$, then
\begin{equation*}\resizebox{1.0\hsize}{!}{
$\begin{aligned}
    &\text{ELBO}_{\theta,\phi}(\alpha^{*}(x,u); x, u)-\text{ELBO}_{\theta,\phi}(0; x,u) \\
    &\geq \frac{-1+\sqrt{1+4\epsilon}}{2} |\Delta_{1-0}(x,u)| + o(|\Delta_{1-0}(x,u)|) + O(\epsilon \log \epsilon)\\
    &\quad \text{as} \quad \text{$|\Delta_{1-0}(x,u)| \to \infty$ and $\epsilon \to 0^{+}$}.
\end{aligned}$}
\end{equation*}
}
That is, if $q_{\phi}(z|x)$ and $q_{\phi}(z|x,u)$ are different so that $\text{SNR}(g)$ is large enough for some $g$, then $\alpha^{*}(x,u)>0$ and our bound is sharper than that of iVAE with positive margins. In the simulation study in Table \hyperref[tab:simulation_alpha_table]{2} in Appendix B.2, $\alpha^{*}(x,u)>0$ holds with positive probability and $\alpha^{*}(x,u)$ can be approximated by a formula based on our theory. Under the same conditions in Theorem \hyperref[thm:positive_margin]{3}, we derived that
\begin{equation}\label{eq:alpha_formula}\resizebox{1.0\hsize}{!}{
$\begin{aligned}
    \alpha_{\text{approx}}^{*}(\epsilon, \Delta_{1-0}(x,u)):= \frac{1-\sqrt{1+4\epsilon}}{2} + \frac{\sqrt{1+4\epsilon}}{1+e^{-\sqrt{1+4\epsilon}\Delta_{1-0}(x,u)}}
\end{aligned}$}
\end{equation}
is in $[0, 1]$ and $\text{ELBO}_{\theta, \phi}(\alpha_{\text{approx}}^{*}(\epsilon, \Delta_{1-0}(x,u)); x, u)-\text{ELBO}_{\theta,\phi}(0;x,u)$ is greater than or equal to the positive margin in Theorem \hyperref[thm:positive_margin]{3}. Details are provided in Lemma \hyperref[lemA.4]{5} in Appendix \hyperref[AppendixA.3]{A.3}. The calculated values by this formula were similar to numerically approximated $\alpha^{*}$, which supports the validity of our theory.

\subsection{Proofs of Propositions}
\subsubsection{Proof of Proposition $1$}
We provide a proof by contradiction. We assume that $q_{\phi}(z|x,u) \neq p_{T,\lambda}(z|u)$ holds w.p. $0$, which is equivalent to assume that the posterior collapse occurs, i.e., $q_{\phi}(z|x,u)=p_{T,\lambda}(z|u)$ holds w.p. $1$. Since $q_{\phi}(z|x,u)=q_{\phi}(z|x)$, it implies that $q_{\phi}(z|x)=p_{T,\lambda}(z|u)$ holds w.p. $1$. Now, $p_{T,\lambda}(z|u)=q_{\phi}(z|x)=\int q_{\phi}(z|x)p(u)du=\int p_{T,\lambda}(z|u)p(u)du$ contradicts to that $p_{T,\lambda}$ is non-trivial.

\subsubsection{Proof of Proposition $2$}
We provide a proof by contradiction.
Since $\mathbb{E}_{q_{\phi}(z|x)}\log p_{f}(x|z) - \mathcal{D}_{\text{KL}}(q_{\phi}(z|x)||p_{T, \lambda}(z|u))$ is equal to $\log p_{\theta}(x|u) - \mathcal{D}_{\text{KL}}(q_{\phi}(z|x)||p_{\theta}(z|x,u))$, $\underset{\phi}{\text{max}}\mathbb{E}_{p(x,u)}\text{ELBO}_{\theta, \phi}(1;x,u)$ is equal to $\mathbb{E}_{p(x,u)}\log p_{\theta}(x|u)$ if and only if $\underset{\phi}{\text{min}}\mathcal{D}_{\text{KL}}(q_{\phi}(z|x)||p_{\theta}(z|x,u))=0$. It implies $p_{\theta}(z|x,u) = p_{\theta}(z|x)$. By Bayes' theorem, $p_{\theta}(z|x,u)=p_{\theta}(x|z,u)p_{T,\lambda}(z|u)/p_{\theta}(x|u)=p_{f}(x|z)p_{T,\lambda}(z|u)/p_{\theta}(x|u)$ and $p_{\theta}(z|x)=p_{f}(x|z)p_{\theta}(z)/p_{\theta}(x)$. Thus, $p_{T, \lambda}(z|u)p_{\theta}(x)=p_{\theta}(z)p_{\theta}(x|u)$. Now, $p_{T, \lambda}(z|u)=\int p_{T, \lambda}(z|u) p_{\theta}(x)dx=\int p_{\theta}(z)p_{\theta}(x|u)dx=p_{\theta}(z)$ contradicts to that the label prior is non-trivial.

\subsubsection{Proof of Proposition $3$}
For any sample $(x, u)$, $\theta=(f, T, \lambda)$, $\phi$, and $\alpha \in [0, 1]$,
\begin{equation*}
\resizebox{1.0\hsize}{!}{
$\begin{aligned}
    &\log p_{\theta}(x|u)\\
    &= \log \Big( \int \frac{p_{\theta}(x, z|u)}{\alpha q_{\phi}(z|x) + (1-\alpha)q_{\phi}(z|x,u)}\Big( \alpha q_{\phi}(z|x) + (1-\alpha)q_{\phi}(z|x,u) \Big)dz \Big)\\
    &\geq \int \Big( \log \frac{p_{\theta}(x, z|u)}{\alpha q_{\phi}(z|x) + (1-\alpha)q_{\phi}(z|x,u)} \Big)\Big( \alpha q_{\phi}(z|x) + (1-\alpha)q_{\phi}(z|x,u) \Big)dz\\
    &= \mathbb{E}_{\alpha q_{\phi}(z|x) + (1-\alpha)q_{\phi}(z|x,u)} [ \log p_{\theta}(x,z|u)-p_{T,\lambda}(z|u)]\\
    &\quad-\mathcal{D}_{\text{KL}}(\alpha q_{\phi}(z|x) + (1-\alpha)q_{\phi}(z|x,u)||p_{T,\lambda}(z|u))
\end{aligned}$}
\end{equation*}
holds by Jensen's inequality. Now, $p_{\theta}(x,z|u)=p_{f}(x|z)p_{T,\lambda}(z|u)$ concludes the proof.

\subsubsection{Proof of Proposition $4$}
By the definition of $\alpha^{*}(x,u)$ and the existence of $\phi$ satisfying $q_{\phi}(z|x,u)=p_{\theta}(z|x,u)$ with probability (w.p.) $1$ w.r.t. $p(x,u)$, 
\begin{equation*}
    \begin{split}
        &\underset{\phi}{\text{max}}\mathbb{E}_{p(x,u)}\text{ELBO}_{\theta, \phi}(\alpha^{*}(x,u);x,u)\\
        &\geq \underset{\phi}{\text{max}}\mathbb{E}_{p(x,u)}\text{ELBO}_{\theta, \phi}(0;x,u)\\
        &= \mathbb{E}_{p(x,u)}\log p_{\theta}(x|u) - \underset{\phi}{\text{min}} \mathcal{D}_{\text{KL}}(q_{\phi}(z|x,u)||p_{\theta}(z|x,u))\\
        &= \mathbb{E}_{p(x,u)}\log p_{\theta}(x|u).
    \end{split}
\end{equation*}
Since $\text{ELBO}_{\theta, \phi}(\alpha^{*}(x,u);x,u)$ is a lower bound of $\log p_{\theta}(x|u)$, the proof is concluded.

\subsubsection{Proof of Proposition $5$}
For any sample $(x, u)$, $\theta=(f, T, \lambda)$, $\phi$, and $\alpha \in [0, 1]$, $\text{ELBO}_{\theta,\phi}(\alpha;x,u)$ can be expressed as
\begin{equation*}\resizebox{1.0\hsize}{!}{
$\begin{aligned}
    &\alpha \mathbb{E}_{q_{\phi}(z|x)} \log p_{f}(x|z) + (1-\alpha)\mathbb{E}_{q_{\phi}(z|x,u)} \log p_{f}(x|z)\\
    &-\int \Big( \log \frac{\alpha q_{\phi}(z|x) + (1-\alpha)q_{\phi}(z|x,u)}{p_{T,\lambda}(z|u)} \Big)(\alpha q_{\phi}(z|x) + (1-\alpha)q_{\phi}(z|x,u))dz\\
    &=\alpha \Big(\text{ELBO}_{\theta,\phi}(1;x,u) + \int \Big( \log \frac{q_{\phi}(z|x)}{p_{T,\lambda}(z|u)} \Big)q_{\phi}(z|x)dz \Big)\\
    &\quad +(1-\alpha) \Big(\text{ELBO}_{\theta,\phi}(0;x,u) + \int \Big( \log \frac{q_{\phi}(z|x,u)}{p_{T,\lambda}(z|u)} \Big)q_{\phi}(z|x,u)dz \Big)\\
    &\quad -\alpha \int \Big( \log \frac{\alpha q_{\phi}(z|x) + (1-\alpha)q_{\phi}(z|x,u)}{p_{T,\lambda}(z|u)} \Big) q_{\phi}(z|x)dz\\
    &\quad-(1-\alpha) \int \Big( \log \frac{\alpha q_{\phi}(z|x) + (1-\alpha)q_{\phi}(z|x,u)}{p_{T,\lambda}(z|u)} \Big) q_{\phi}(z|x,u)dz\\
    &=\alpha \text{ELBO}_{\theta,\phi}(1;x,u) + (1-\alpha)\text{ELBO}_{\theta,\phi}(0;x,u) \\
    &\quad +\alpha \mathcal{D}_{\text{KL}}(q_{\phi}(z|x)||\alpha q_{\phi}(z|x) + (1-\alpha)q_{\phi}(z|x,u))\\
    &\quad + (1-\alpha)\mathcal{D}_{\text{KL}}(q_{\phi}(z|x,u)||\alpha q_{\phi}(z|x)+ (1-\alpha)q_{\phi}(z|x,u))\\
    &=\alpha \text{ELBO}_{\theta,\phi}(1;x,u)+(1-\alpha)\text{ELBO}_{\theta,\phi}(0;x,u)\\
    &\quad +\alpha \mathcal{D}^{1-\alpha}_{\text{skew}}(q_{\phi}(z|x)||q_{\phi}(z|x,u))+(1-\alpha) \mathcal{D}^{\alpha}_{\text{skew}}(q_{\phi}(z|x,u)||q_{\phi}(z|x)).
\end{aligned}$}
\end{equation*}

\subsubsection{Proof of Proposition 6}
For any sample $(x, u)$ and $\alpha \in [0, 1]$, the first derivative of $\text{ELBO}_{\theta, \phi}(\alpha;x,u)$ can be expressed as
\begin{equation*}\resizebox{1.0\hsize}{!}{
$\begin{aligned}
    &\text{ELBO}_{\theta, \phi}^{(1)}(\alpha;x,u)\\
    &=\mathbb{E}_{q_{\phi}(z|x)} \log p_{f}(x|z) - \mathbb{E}_{q_{\phi}(z|x,u)} \log p_{f}(x|z) \\
    &\quad- \frac{d}{d\alpha} \int \Big( \log \frac{\alpha q_{\phi}(z|x) + (1-\alpha)q_{\phi}(z|x,u)}{p_{T,\lambda}(z|u)} \Big) (\alpha q_{\phi}(z|x) + (1-\alpha)q_{\phi}(z|x,u))dz \\
    &=\Big( \text{ELBO}_{\theta, \phi}(1;x,u)+\mathcal{D}_{\text{KL}}(q_{\phi}(z|x)||p_{T,\lambda}(z|u)) \Big)\\
    &\quad- \Big( \text{ELBO}_{\theta, \phi}(0;x,u)+\mathcal{D}_{\text{KL}}(q_{\phi}(z|x,u)||p_{T,\lambda}(z|u)) \Big)\\
    &\quad -\Big[ \int \frac{q_{\phi}(z|x)-q_{\phi}(z|x,u)}{\alpha q_{\phi}(z|x) + (1-\alpha)q_{\phi}(z|x,u)}(\alpha q_{\phi}(z|x) + (1-\alpha)q_{\phi}(z|x,u))dz \\
    &\quad+\int \Big( \log \frac{\alpha q_{\phi}(z|x) + (1-\alpha)q_{\phi}(z|x,u)}{p_{T,\lambda}(z|u)} \Big)(q_{\phi}(z|x)-q_{\phi}(z|x,u))dz \Big]\\
    &=\text{ELBO}_{\theta,\phi}(1;x,u)-\text{ELBO}_{\theta,\phi}(0;x,u)\\
    &\quad+\mathcal{D}_{\text{KL}}(q_{\phi}(z|x)||\alpha q_{\phi}(z|x) + (1-\alpha)q_{\phi}(z|x,u)) \\
    &\quad-\mathcal{D}_{\text{KL}}(q_{\phi}(z|x,u)||\alpha q_{\phi}(z|x) + (1-\alpha)q_{\phi}(z|x,u)).
\end{aligned}$}
\end{equation*}
With this, the second derivative of $\text{ELBO}_{\theta, \phi}(\alpha;x,u)$ can be expressed as
\begin{equation*}
\resizebox{1.0\hsize}{!}{
$\begin{aligned}
    &\text{ELBO}_{\theta,\phi}^{(2)}(\alpha;x,u)\\
    &= -\frac{d}{d\alpha} \int \Big( \log \frac{\alpha q_{\phi}(z|x) + (1-\alpha) q_{\phi}(z|x,u)}{p_{T, \lambda}(z|u)} \Big) (q_{\phi}(z|x)-q_{\phi}(z|x,u))dz\\
    &= - \int \Big( \log \frac{q_{\phi}(z|x)-q_{\phi}(z|x,u)}{\alpha q_{\phi}(z|x) + (1-\alpha) q_{\phi}(z|x,u)} \Big) (q_{\phi}(z|x)-q_{\phi}(z|x,u))dz\\
    &= -\int \frac{(q_{\phi}(z|x)-q_{\phi}(z|x,u))^{2}}{\alpha q_{\phi}(z|x) + (1-\alpha)q_{\phi}(z|x,u)}dz.
\end{aligned}$
} \end{equation*}
Thus, $\text{ELBO}_{\theta,\phi}^{(2)}(\alpha;x,u) \leq 0$ for any $\alpha \in [0, 1]$. The equality holds if and only if $q_{\phi}(z|x)=q_{\phi}(z|x,u)$ for all $z$, which concludes the proof.

\subsection{Proofs of Theorems}\label{AppendixA.3}
\subsubsection{Proof of Theorem $1$}
%yk: or, for a given non-trivial label prior, we can derive that iVAEs occur posterior collapse.
The proof is extended from the proof of Proposition 2 in \cite{dai2020usual}.
We first provide detailed formulations of conditions for Theorem $1$, and then derive Theorem $1$. We denote the parameter in decoder networks by $\psi$. The reconstruction error at $x$ with posterior $N(\mu, \sigma)$ can be expressed as $\mathbb{E}_{Z \sim N(\mu, \sigma^{2})} \lvert\lvert x - f(Z; \psi) \rvert\rvert_{2}^{2}$. We formulate conditions $(C1)$ and $(C2)$ \citep{dai2020usual}.\\ \\
$(C1)$ $\frac{\partial}{\partial\mu} \mathbb{E}_{Z \sim N(\mu, \sigma^{2})} \lvert\lvert x - f(Z; \psi) \rvert\rvert_{2}^{2}$ and $\frac{\partial}{\partial\sigma} \mathbb{E}_{Z \sim N(\mu, \sigma^{2})} \lvert\lvert x - f(Z; \psi) \rvert\rvert_{2}^{2}$ are $L$-Lipschitz continuous.\\
$(C2)$ $\frac{\partial}{\partial\sigma} \mathbb{E}_{Z \sim N(\mu, \sigma^{2})} \lvert\lvert x - f(Z; \psi) \rvert\rvert_{2}^{2} \geq c$ for some $c>0$.\\ \\
The $(C1)$ means that the reconstruction error is sufficiently smooth and its partial derivatives have bounded slopes w.r.t. $(\mu, \sigma)$. The $(C2)$ means that the decoder increases the reconstruction error as the uncertainty of latent variables increases. The $(C2)$ does not hold when the decoder is degenerated, i.e., $f(z; \psi)$ is constant.

Now, we derive Theorem $1$. We show that for any iVAEs satisfying $(C1)$ and $(C2)$, there is a posterior collapse case having larger value of ELBO. As in \cite{dai2020usual}, we consider that the observation noise $\epsilon$ follows a Gaussian distribution $N(0, \gamma I_{d_{X}})$. We reparametrize the decoder network with a scale parameter $w \in [0, 1]$ and denote the parameter for the decoder by $\psi = (w, \psi \backslash w)$. The output of the decoder with $\psi$ can be expressed as $f(z; \psi)=f(wz;\psi \backslash w)$. We denote means and standard deviations of posterior distributions by
\begin{equation*}
    \begin{split}
        &m_{Z|X,U}(\{(x_{i}, u_{i})\}_{i=1}^{n}; \phi)\\
        &:=(\mu_{Z|X,U}(x_{1}, u_{1};\phi), \dots, \mu_{Z|X,U}(x_{n}, u_{n};\phi))
    \end{split}
\end{equation*}
and
\begin{equation*}
    \begin{split}
        &s_{Z|X,U}(\{(x_{i}, u_{i})\}_{i=1}^{n}; \phi)\\
        &:=(\sigma_{Z|X,U}(x_{1},u_{1};\phi), \dots, \sigma_{Z|X,U}(x_{n},u_{n};\phi)),
    \end{split}
\end{equation*}
respectively, where $\mu_{Z|X,U}(x_{i}, u_{i};\phi)$ and $\sigma_{Z|X,U}(x_{i},u_{i};\phi)$ denote the mean and standard deviations of posterior distributions at $i$-th datum, respectively.

For any iVAEs with decoder parameter $\tilde{\psi}$ satisfying $(C1)$ and $(C2)$, posterior parameter $\tilde{\phi}$, and label prior parameter $(\tilde{T}, \tilde{\lambda})$, values evaluated at $\tilde{\phi}$ are denoted by $\tilde{m}_{Z|X,U}$ and $\tilde{s}_{Z|X,U}$. For simplicity, we use $m_{Z|X,U}$, $s_{Z|X,U}$, $\tilde{m}_{Z|X,U}$, and $\tilde{s}_{Z|X,U}$ if there is no confusion and use $m_{Z|X, U,i}$, $s_{Z|X, U, i}$, $\tilde{m}_{Z|X, U, i}$, and $\tilde{s}_{Z|X, U, i}$, respectively, to indicate their $i$-th components. In a similar manner, means and standard deviations by label prior networks are denoted by $m_{Z|U}$ and $s_{Z|U}$, respectively. We denote $(m^{\text{Scale}}_{Z|X,U,i})_{j}=((m_{Z|X,U,i})_{j}-(\tilde{m}_{Z|U,i})_{j})/(\tilde{s}_{Z|U,i})_{j}$, $(s^{\text{Scale}}_{Z|X,U,i})_{j}=(s_{Z|X,U,i})_{j}/(\tilde{s}_{Z|U,i})_{j}$, $(\tilde{m}^{\text{Scale}}_{Z|X,U,i})_{j}=((\tilde{m}_{Z|X,U,i})_{j}-(\tilde{m}_{Z|U,i})_{j})/(\tilde{s}_{Z|U,i})_{j}$, and $(\tilde{s}^{\text{Scale}}_{Z|X,U,i})_{j}=(\tilde{s}_{Z|X,U,i})_{j}/(\tilde{s}_{Z|U,i})_{j}$. With these terms, we can express the reconstruction error at the $i$-th datum as
\begin{equation*}
\begin{split}
    &r(w(m^{\text{Scale}}_{Z|X,U,i}+\tilde{m}_{Z|U,i}/\tilde{s}_{Z|U,i}), ws^{\text{Scale}}_{Z|X,U,i}, \tilde{\psi}\backslash\tilde{w}, x_{i})\\
    &=\mathbb{E}_{\epsilon \sim N(0, I_{d})} \lVert x_{i} - f\big( \tilde{s}_{Z|U}(w(m^{\text{Scale}}_{Z|X,U,i}+\tilde{m}_{Z|U,i}/\tilde{s}_{Z|U,i}))\\
    &\quad +\tilde{s}_{Z|U,i}(ws^{\text{Scale}}_{Z|X,U,i}) \epsilon; \tilde{\psi} \backslash \tilde{w}\big)\rVert_{2}^{2}
\end{split}
\end{equation*}
and the average reconstruction error by $\bar{r}(w(m^{\text{Scale}}_{Z|X,U}+\tilde{m}_{Z|U}/\tilde{s}_{Z|U}), ws^{\text{Scale}}_{Z|X,U})$. Here, all the parameters in decoder and label prior but $w$ are fixed. Then, the average of the negative ELBO of iVAE evaluated with $\{(x_{i}, u_{i})\}_{i=1}^{n}$ is the same as 
\begin{equation*}
\begin{split}
    &h(m_{Z|X,U}, s_{Z|X,U}, w)\\ &:=\gamma^{-1}\bar{r}(w(m^{\text{Scale}}_{Z|X,U}+\tilde{m}_{Z|U}/\tilde{s}_{Z|U}), ws^{\text{Scale}}_{Z|X,U})+d_{X}\log \gamma \\
    &\quad + n^{-1}\sum_{i=1}^{n}2\mathcal{D}_{\text{KL}}(N(m^{\text{Scale}}_{Z|X,U,i}, (s^{\text{Scale}}_{Z|X,U,i})^{2})||N(0, 1))
\end{split}
\end{equation*}
up to constant addition and multiplication since $\mathcal{D}_{\text{KL}}(N(m_{Z|X,U,i}, s_{Z|X,U,i}^{2})||N(\tilde{m}_{Z|U,i}, \tilde{s}^{2}_{Z|U,i}))=\mathcal{D}_{\text{KL}}(N(m^{\text{Scale}}_{Z|X,U,i}, (s^{\text{Scale}}_{Z|X,U,i})^{2})||N(0, 1))$. We define $h^{\text{appr}}$, an approximation of an upper bound of $h$ based on the Taylor series of $\bar{r}$ in Lemma \hyperref[lem:h_appr]{1}. The $h^{\text{appr}}$ is equal to $h$ when $(m_{Z|X,U}, s_{Z|X,U}, w)=(\tilde{m}_{Z|X,U}, \tilde{s}_{Z|X,U}, \tilde{w})$ and is an upper bound of $h$ when $w(s_{Z|X,U,i})_{j} \in \{0, \tilde{w}(\tilde{s}_{Z|X,U, i})_{j} \}$ for all $i=1,\dots,n$ and $j=1,\dots,d_{Z}$.

\begin{lemma}\label{lem:h_appr} For any $\{(x_{i}, u_{i})\}_{i=1}^{n}$, we define 
\begin{align*}
    & h^{\text{appr}}(m_{Z|X,U}, s_{Z|X,U}, w;\tilde{m}_{Z|X,U}, \tilde{s}_{Z|X,U}, \tilde{w})\\
    &:=\gamma^{-1}\bar{r}^{\text{appr}}(w(m^{\text{Scale}}_{Z|X,U}+\tilde{m}_{Z|U}/\tilde{s}_{Z|U}), ws^{\text{Scale}}_{Z|X,U};\\
    &\quad \tilde{w}(\tilde{m}^{\text{Scale}}_{Z|X,U}+\tilde{m}_{Z|U}/\tilde{s}_{Z|U}), w\tilde{s}^{\text{Scale}}_{Z|X,U})+d_{X}\log \gamma\\
    &\quad+n^{-1}\sum_{i=1}^{n}2\mathcal{D}_{\text{KL}}(N(m^{\text{Scale}}_{Z|X,U,i}, (s^{\text{Scale}}_{Z|X,U,i})^{2})||N(0, 1))
\end{align*}
where $\bar{r}^{\text{appr}}(u, v; \tilde{u}, \tilde{v}):=\bar{r}(\tilde{u}, \tilde{v}) + (u-\tilde{u})^{T}\frac{\partial}{\partial u}\bar{r}(\tilde{u}, \tilde{v}) + \frac{L}{2} \lvert\lvert u-\tilde{u} \rvert\rvert_{2}^{2} + \sum_{j=1}^{D} g^{\text{appr}} \Big( v_{j}, \tilde{v}_{j}, \frac{\partial}{\partial v_{j}}\bar{r}(\tilde{u}, \tilde{v}) \Big)$ for any $u$, $v$, $\tilde{u}$ and $\tilde{v} \in \mathbb{R}^{D}$ and $g^{\text{appr}}: \mathbb{R}^{3} \to \mathbb{R}$ is defined as follows:
\begin{equation}\label{eq:g_appr}
\resizebox{1.0\hsize}{!}{
$\begin{aligned}
    &g^{\text{appr}}(v, \tilde{v}, \delta)\\
    &= \begin{cases} 
      -\frac{\delta^{2}}{2L} + \frac{\delta^{2}}{2L\tilde{v}^{2}}v^{2} & \text{if }v \geq \tilde{v}-\frac{\delta}{L}\text{ and } \{v,\tilde{v},\delta\}\geq0, \\
      \big(\frac{L\tilde{v}^{2}}{2}-\delta \tilde{v} \big)+\big(\frac{\delta}{\tilde{v}}-\frac{L}{2} \big)v^{2} & \text{if }v < \tilde{v}-\frac{\delta}{L}\text{ and } \{v,\tilde{v},\delta\}\geq0, \\
      \infty & \text{otherwise.}
    \end{cases}
\end{aligned}$}
\end{equation}
Then, 
\begin{equation*}
    \begin{split}
        & h^{\text{appr}}(\tilde{m}_{Z|X,U}, \tilde{s}_{Z|X,U}, \tilde{w}; \tilde{m}_{Z|X,U}, \tilde{s}_{Z|X,U}, \tilde{w}) \\
        & = h(\tilde{m}_{Z|X,U}, \tilde{s}_{Z|X,U}, \tilde{w})
    \end{split}
\end{equation*}
and
\begin{equation*}
    \begin{split}
        & h^{\text{appr}}(m_{Z|X,U}, s_{Z|X,U}, w; \tilde{m}_{Z|X,U}, \tilde{s}_{Z|X,U}, \tilde{w}) \\
        & \geq h(m_{Z|X,U}, s_{Z|X,U}, w)
    \end{split}
\end{equation*}
if $w(s_{Z|X,U, i})_{j} \in \{0, \tilde{w}(\tilde{s}_{Z|X,U, i})_{j} \}$ for all $i=1,\dots,n$ and $j=1,\dots,d_{Z}$.
\end{lemma}
\begin{proof}[Proof of Lemma 1] Section $3$ in the supplementary file of \cite{dai2020usual} showed that $\bar{r}^{\text{appr}}(\tilde{u}, \tilde{v}; \tilde{u}, \tilde{v})=\bar{r}(\tilde{u}, \tilde{v})$ and $\bar{r}^{\text{appr}}(u, v; \tilde{u}, \tilde{v}) \geq \bar{r}(u, v)$ if $v_{j} \in \{0, \tilde{v}_{j}\}$ for all $j$. It concludes the proof since the difference of $h^{\text{appr}}$ from $h$ is changing $\bar{r}$ to $\bar{r}^{\text{appr}}$.
\end{proof}
Let $c_{i, j}$ be coefficients of $v^{2}$ in \eqref{eq:g_appr} determined by $(v, \tilde{v}, \delta)=(w(s^{\text{Scale}}_{Z|X,U, i})_{j}, \tilde{w}(\tilde{s}^{\text{Scale}}_{Z|X,U, i})_{j}, \tilde{\delta}_{i,j})$ where $\tilde{\delta}_{i,j}:=\nabla\bar{r}(\tilde{w}(\tilde{m}^{\text{Scale}}_{Z|X,U}+\tilde{m}_{Z|U}/\tilde{s}_{Z|U}), \tilde{w}\tilde{s}^{\text{Scale}}_{Z|X,U})_{(n+i-1)d_{Z}+j}$.
The $c_{i,j}$ is positive and finite since $w(s^{\text{Scale}}_{Z|X, U, i})_{j} \geq 0$, $\tilde{w}(\tilde{s}^{\text{Scale}}_{Z|X, U, i})_{j} \geq 0$, and $0 < \tilde{\delta}_{i,j} \leq L$ by $(C1)$ and $(C2)$.

We denote the minimizer of $h^{\text{appr}}$ by $(m^{*}_{Z|X,U}(w), s^{*}_{Z|X,U}(w))$. Since
\begin{equation*}
\resizebox{1.0\hsize}{!}{
$\begin{aligned}
    & h^{\text{appr}}(m_{Z|X,U}, s_{Z|X,U}, w;\tilde{m}_{Z|X,U}, \tilde{s}_{Z|X,U}, \tilde{w})\\
    &=\text{Const.}+ n^{-1}\sum_{i=1}^{n} \sum_{j=1}^{d_{Z}} \bigg( \gamma^{-1} \Big( w(m^{\text{Scale}}_{Z|X, U, i}+\tilde{m}_{Z|U,i}/\tilde{s}_{Z|U,i})_{j} \tilde{\delta}_{i,j}\\
    &\quad+\frac{L}{2} \big(w^{2}(m^{\text{Scale}}_{Z|X, U, i}+\tilde{m}_{Z|U,i}/\tilde{s}_{Z|U,i})_{j}^{2}\\
    &\quad- 2w(m^{\text{Scale}}_{Z|X, U, i}+\tilde{m}_{Z|U,i}/\tilde{s}_{Z|U,i})_{j}\tilde{w}(\tilde{m}^{\text{Scale}}_{Z|X, U, i}+\tilde{m}_{Z|U,i}/\tilde{s}_{Z|U,i})_{j}\big)\\
    &\quad+c_{i,j}w^{2} (s^{\text{Scale}}_{Z|X, U, i})_{j}^{2}\Big) + (m^{\text{Scale}}_{Z|X, U, i})_{j}^{2} + (s^{\text{Scale}}_{Z|X, U, i})_{j}^{2}-\log(s^{\text{Scale}}_{Z|X, U, i})_{j}^{2} \bigg),
\end{aligned}$
}
\end{equation*}
$h^{\text{appr}}$ is a quadratic function w.r.t. $(m^{\text{Scale}}_{Z|X, U, i}+\tilde{m}_{Z|U,i}/\tilde{s}_{Z|U,i})_{j}$ and coefficients of second-order and first-order terms are $n^{-1}(\gamma^{-1}Lw^{2}/2+1)$ and $n^{-1}\big( \gamma^{-1}w(\tilde{\delta}_{i,j}-L\tilde{w}(\tilde{m}_{Z|X,U,i})_{j}) - 2 \tilde{m}_{Z|U,i}/\tilde{s}_{Z|U,i} \big)$, respectively.
This implies
\begin{equation*}
\resizebox{1.0\hsize}{!}{
$\begin{aligned}
    &(m^{*}_{Z|X,U, i}(w))_{j}\\
    &=\bigg( w(\tilde{s}_{Z|U,i})_{j}\big( L\tilde{w}(\tilde{m}_{Z|X,U,i})_{j} - \tilde{\delta}_{i,j} \big) +2\gamma (\tilde{m}_{Z|U,i})_{j} \bigg)/(2\gamma+Lw^{2}).
\end{aligned}$}
\end{equation*}
For $(s_{Z|X,i})_{j}$, $\partial h^{\text{appr}}/\partial(s^{\text{Scale}}_{Z|X, i})_{j}^{2}=n^{-1}(\gamma^{-1} c_{i,j}w^{2} + 1 - 1/(s^{\text{Scale}}_{Z|X, i})_{j}^{2})$ and $\partial^{2} h^{\text{appr}}/\partial\big( (s^{\text{Scale}}_{Z|X, i})_{j}^{2}\big)^{2}=n^{-1}/(s^{\text{Scale}}_{Z|X, i})_{j}^{4}>0$ imply $(s^{*}_{Z|X,U, i}(w))_{j}^{2}=(\tilde{s}_{Z|U,i})_{j}^{2}(\gamma^{-1}c_{i,j}w^{2}+1)^{-1}$.
By substituting $(m^{*}_{Z|X,U}(w), s^{*}_{Z|X,U}(w))$, we have
\begin{equation*}
\begin{split}
    &\partial h^{\text{appr}}(m^{*}_{Z|X}(w), s^{*}_{Z|X}(w), w;\tilde{m}_{Z|X}, \tilde{s}_{Z|X}, \tilde{w})/\partial w^{2}\\
    &= n^{-1}\sum_{i=1}^{n}\sum_{j=1}^{d_{Z}}\bigg(\frac{c_{i,j}}{\gamma + c_{i,j}w^{2}} + O(\gamma^{-2})\bigg).
\end{split}
\end{equation*}
Since $c_{i,j}$ is positive, this partial derivative is positive for all $w \in [0, 1]$ when $\gamma$ is sufficiently large. In this case, the optimal $w$ is zero and $\big( (m^{*}_{Z|X, U,i})_{j}(0), (s^{*}_{Z|X, U,i})_{j}(0) \big)=((\tilde{m}_{Z|U,i})_{j}, (\tilde{s}_{Z|U,i})_{j})$, i.e., posterior collapse cases.
%Note that this global optima is feasible since for any $(T, \lambda)$, there is $\phi$ satisfying $q_{\phi}(z|x,u)=p_{T, \lambda}(z|u)$ for all $(x,u)$.
By Lemma \hyperref[lem:h_appr]{1}, $h(\tilde{m}_{Z|U}, \tilde{s}_{Z|U}, 0) \leq h^{\text{appr}}(\tilde{m}_{Z|U}, \tilde{s}_{Z|U}, 0;\tilde{m}_{Z|X,U}, \tilde{s}_{Z|X,U}, \tilde{w})$ and $h^{\text{appr}}(\tilde{m}_{Z|X,U}, \tilde{s}_{Z|X,U}, \tilde{w};\tilde{m}_{Z|X,U}, \tilde{s}_{Z|X,U}, \tilde{w})=h(\tilde{m}_{Z|X,U}, \tilde{s}_{Z|X,U}, \tilde{w})$. Since $(\tilde{m}_{Z|U}, \tilde{s}_{Z|U}, 0)$ is the global optima of $h^{\text{appr}}$, $h(\tilde{m}_{Z|U}, \tilde{s}_{Z|U}, 0) < h(\tilde{m}_{Z|X,U}, \tilde{s}_{Z|X,U}, \tilde{w})$. That is, there is a posterior collapse case whose value of ELBO is better than current networks. Thus, the iVAEs are worse than the posterior collapse case.

\subsubsection{Proof of Theorem $2$}
We provide a proof by contradiction. Let $q_{\phi}(z|x,u)=p_{T,\lambda}(z|u)$ holds w.p. $1$. Since $\alpha^{*}(x,u)>0$ and $\alpha^{*}(x,u)q_{\phi}(z|x)+(1-\alpha^{*}(z|x,u)q_{\phi}(z|x,u)=p_{T,\lambda}(z|u)$ w.p. $1$, we have $q_{\phi}(z|x)=p_{T, \lambda}(z|u)$, which contradicts to that $p_{T,\lambda}(z|u)$ is non-trivial by Proposition $1$.

\subsubsection{Proof of Theorem \hyperref[thm:posivier_margin]{3}}
We first provide lemmas with proofs, and then derive the theorem.
\begin{lemma}\label{lemA.1} (Equation (18) in \cite{nishiyama2019new}) For any $t \in [0, 1]$, real-valued function $g$, and probability density functions $p(z)$ and $q(z)$, $\partial\mathcal{D}_{\text{KL}}(p(z)||(1-t) p(z) + t q(z))/\partial t$ is greater than or equal to $t(\mathbb{E}_{q(z)}[g(z)]-\mathbb{E}_{p(z)}[g(z)])^{2}/\big( t(1-t)(\mathbb{E}_{q(z)}[g(z)]-\mathbb{E}_{p(z)}[g(z)])^{2} + (1-t) Var_{p(z)}[g(z)] + t Var_{q(z)}[g(z)] \big)$.
\end{lemma}

\begin{lemma}\label{lemA.2} For any datum $(x,u)$ and positive number $\epsilon$, if there is a function $g: \mathcal{Z} \to \mathbb{R}$ satisfying $\text{SNR}(g) \geq 1/\epsilon$, then
$$\text{ELBO}_{\theta, \phi}(\alpha; x, u)-\text{ELBO}_{\theta, \phi}(0; x, u) \geq \text{LB}(\alpha, \epsilon, \Delta_{1-0}(x, u))$$
where $\text{LB}(\alpha, \epsilon, \Delta_{1-0}(x, u)):=\alpha \Delta_{1-0}(x, u) + \alpha \int_{\alpha}^{1} (1-t)/(t(1-t) + \epsilon) dt + (1-\alpha) \int_{0}^{\alpha}t/(t(1-t) + \epsilon)dt$.
\end{lemma}
\begin{proof}[Proof of Lemma 3]
By Proposition 5,
\begin{equation}\label{eqn:lem2_1}
\resizebox{1.0\hsize}{!}{
$\begin{aligned}
    &\text{ELBO}_{\theta, \phi}(\alpha;x,u)-\text{ELBO}_{\theta, \phi}(0; x, u)\\
    &=\alpha \Delta_{1-0}(x,u) + \alpha \mathcal{D}_{\text{KL}}(q_{\phi}(z|x)||\alpha q_{\phi}(z|x) + (1-\alpha) q_{\phi}(z|x,u))\\
    &\quad+ (1-\alpha) \mathcal{D}_{\text{KL}}(q_{\phi}(z|x,u)||\alpha q_{\phi}(z|x) + (1-\alpha) q_{\phi}(z|x,u)).
\end{aligned}$}
\end{equation}
By substituting $q_{\phi}(z|x)$ and $q_{\phi}(z|x,u)$ to $q$ and $p$ in Lemma \hyperref[lemA.1]{2}, respectively, and integrating both sides from $0$ to $\alpha$, we have
\begin{equation*}
\resizebox{1.0\hsize}{!}{
$\begin{aligned}
    &\mathcal{D}_{\text{KL}}(q_{\phi}(z|x,u)||\alpha q_{\phi}(z|x) + (1-\alpha) q_{\phi}(z|x,u)) \\
    &\geq \int_{0}^{\alpha} \bigg( \big( t(\mathbb{E}_{q_{\phi}(z|x,u)}[g(z)]-\mathbb{E}_{q_{\phi}(z|x)}[g(z)])^{2} \big)/\big( t(1-t)(\mathbb{E}_{q_{\phi}(z|x,u)}[g(z)]-\mathbb{E}_{q_{\phi}(z|x)}[g(z)])^{2} \\
    & \quad + (1-t) Var_{q_{\phi}(z|x)}[g(z)] + t Var_{q_{\phi}(z|x,u)}[g(z)] \big) \bigg)dt.
\end{aligned}$}
\end{equation*}
Since $\text{SNR}(g) \geq 1/\epsilon$,
\begin{equation}\label{eqn:lem2_2}
\resizebox{1.0\hsize}{!}{
$\begin{aligned}
\mathcal{D}_{\text{KL}}(q_{\phi}(z|x,u)||\alpha q_{\phi}(z|x) + (1-\alpha) q_{\phi}(z|x,u)) \geq \int_{0}^{\alpha} \frac{t}{t(1-t) + \epsilon}dt.\end{aligned}$}
\end{equation}
In a similar way, we can derive
\begin{equation}\label{eqn:lem2_3}
\resizebox{1.0\hsize}{!}{
$\begin{aligned}
\mathcal{D}_{\text{KL}}(q_{\phi}(z|x)||\alpha q_{\phi}(z|x) + (1-\alpha) q_{\phi}(z|x,u)) \geq \int_{\alpha}^{1} \frac{1-t}{t(1-t) + \epsilon}dt.\end{aligned}$}
\end{equation}
By \eqref{eqn:lem2_1}, \eqref{eqn:lem2_2}, and \eqref{eqn:lem2_3}, the proof is concluded.
\end{proof}

\begin{lemma}\label{lemA.3} The first and second partial derivatives of $\text{LB}(\alpha, \epsilon, \Delta_{1-0}(x, u))$ w.r.t. $\alpha$, respectively, can be expressed as
\begin{equation*}
\resizebox{1.0\hsize}{!}{
$\begin{aligned}
    \frac{\partial\text{LB}(\alpha, \epsilon, \Delta_{1-0}(x, u))}{\partial \alpha} = \Delta_{1-0}(x, u) + \frac{1}{\sqrt{1+4\epsilon}}\log\abs{\frac{\alpha - \frac{1+\sqrt{1+4\epsilon}}{2}}{\alpha - \frac{1-\sqrt{1+4\epsilon}}{2}}}
\end{aligned}$
}
\end{equation*}
and $\partial^{2}\text{LB}(\alpha, \epsilon, \Delta_{1-0}(x, u))/\partial \alpha^{2} = -1/\big(\alpha(1-\alpha)+\epsilon\big)$.
Thus, $\text{LB}(\alpha, \epsilon, \Delta_{1-0}(x, u))$ is strictly concave w.r.t. $\alpha$ when $\alpha \in [0, 1]$.
\end{lemma}
\begin{proof}[Proof of Lemma 4]
The  $\partial \text{LB}(\alpha, \epsilon, \Delta_{1-0}(x, u))/\partial \alpha$ can be expressed as $\Delta_{1-0}(x, u) + \int_{\alpha}^{1} (1-t)/(t(1-t) + \epsilon) dt - \int_{0}^{\alpha}t/(t(1-t) + \epsilon)dt$. Let $t_{+}=\frac{1+\sqrt{1+4\epsilon}}{2}$ and $t_{-}=\frac{1-\sqrt{1+4\epsilon}}{2}$. Since $\int \frac{t}{t(1-t)+\epsilon}dt=-\frac{t_{+}}{t_{+}-t_{-}}\log \abs{t-t_{+}} + \frac{t_{-}}{t_{+}-t_{-}} \log \abs{t-t_{-}} + C$ where $C$ is the constant of integration and $\int (1-t)/(t(1-t)+\epsilon) dt = \int t/(t(1-t)+\epsilon) dt$, we can derive
\begin{equation*}
\resizebox{1.0\hsize}{!}{
$\begin{aligned}
\frac{\partial\text{LB}(\alpha, \epsilon, \Delta_{1-0}(x, u))}{\partial \alpha} = \Delta_{1-0}(x, u) + \frac{1}{\sqrt{1+4\epsilon}}\log\abs{\frac{\alpha - \frac{1+\sqrt{1+4\epsilon}}{2}}{\alpha - \frac{1-\sqrt{1+4\epsilon}}{2}}}.
\end{aligned}$
}
\end{equation*}
By differentiating the first derivative w.r.t. $\alpha$ again, we have $\partial^{2}\text{LB}(\alpha, \epsilon, \Delta_{1-0}(x, u))/\partial \alpha^{2} = -1/\big( \alpha(1-\alpha)+\epsilon \big)$.
\end{proof}

\begin{lemma}\label{lemA.4}The maximizer of $\text{LB}(\alpha, \epsilon, \Delta_{1-0}(x, u))$ over $\alpha \in [0, 1]$ is $\alpha^{*}_{\text{approx}}(\epsilon, \Delta_{1-0}(x,u)):= \frac{1-\sqrt{1+4\epsilon}}{2} + \frac{\sqrt{1+4\epsilon}}{1+e^{-\sqrt{1+4\epsilon}\Delta_{1-0}(x,u)}}$ if and only if $\abs{\Delta_{1-0}(x,u)} \leq \frac{1}{\sqrt{1+4\epsilon}} \log \frac{(\sqrt{1+4\epsilon}+1)^{2}}{4\epsilon} = -\log\epsilon+O(\epsilon \log\epsilon)$ as $\epsilon \to 0^{+}$.
\end{lemma}
\begin{proof}[Proof of Lemma 5]
By Lemma \hyperref[lemA.3]{4}, the first partial derivative of $\text{LB}(\alpha, \epsilon, \Delta_{1-0}(x, u))$ w.r.t. $\alpha$ is zero if and only if $\alpha=\alpha_{\text{approx}}^{*}(\epsilon, \Delta_{1-0}(x,u))$. The solution of $\alpha_{\text{approx}}^{*}(\epsilon, \Delta_{1-0}(x,u)) \in [0, 1]$ can be expressed as $\abs{\Delta_{1-0}(x,u)} \leq \frac{1}{\sqrt{1+4\epsilon}} \log \frac{(\sqrt{1+4\epsilon}+1)^{2}}{4\epsilon}$.
Next, we prove $\frac{1}{\sqrt{1+4\epsilon}} \log \frac{(\sqrt{1+4\epsilon}+1)^{2}}{4\epsilon} = -\log\epsilon+O(\epsilon \log\epsilon)$ as $\epsilon \to 0^{+}$.
We have
\begin{equation*}
\resizebox{0.8\hsize}{!}{
$\begin{aligned}
&\Big( \frac{1}{\sqrt{1+4\epsilon}} \log \frac{(\sqrt{1+4\epsilon}+1)^{2}}{4\epsilon} - (-\log \epsilon) \Big) \frac{1}{\epsilon\log\epsilon}\\
&= \frac{4}{\sqrt{1+4\epsilon}(\sqrt{1+4\epsilon}+1)}-\frac{2}{\sqrt{1+4\epsilon}}\frac{\log 2/(\sqrt{1+4\epsilon}+1)}{\epsilon \log \epsilon}.\end{aligned}$}
\end{equation*}
Here, the first term in RHS converges to $2$ as $\epsilon \to 0^{+}$ and, by L'Hospital's rule, the limit of the second term is $\underset{\epsilon \to 0^{+}}{\text{lim}}\frac{\log 2/(\sqrt{1+4\epsilon}+1)}{\epsilon \log \epsilon} = \underset{\epsilon \to 0^{+}}{\text{lim}}\frac{-2/(\sqrt{1+4\epsilon}+1)\sqrt{1+4\epsilon}}{\log \epsilon + 1}=0$, which concludes the proof.
\end{proof}

Now, we prove Theorem \hyperref[thm:posivier_margin]{3}. Let $t_{+}=\frac{1+\sqrt{1+4\epsilon}}{2}$ and $t_{-}=\frac{1-\sqrt{1+4\epsilon}}{2}$. Since $\int \frac{t}{t(1-t)+\epsilon}dt=-\frac{t_{+}}{t_{+}-t_{-}}\log \abs{t-t_{+}} + \frac{t_{-}}{t_{+}-t_{-}} \log \abs{t-t_{-}} + C$ where $C$ is the constant of integration and $\int (1-t)/(t(1-t)+\epsilon) dt = \int t/(t(1-t)+\epsilon) dt$, we can derive
\begin{equation*}
\resizebox{1.0\hsize}{!}{
$\begin{aligned}
    &\text{LB}(\alpha, \epsilon, \Delta_{1-0}(x,u))\\
    &=\alpha \Delta_{1-0}(x,u) -\frac{\alpha t_{+}}{t_{+}-t_{-}} \log \abs{\alpha-1+t_{+}}+\frac{\alpha t_{-}}{t_{+}-t_{-}} \log \abs{\alpha-1+t_{-}}\\
    &\quad-\frac{(1-\alpha) t_{+}}{t_{+}-t_{-}} \log \abs{\alpha-t_{+}}+\frac{(1-\alpha) t_{-}}{t_{+}-t_{-}} \log \abs{\alpha-t_{-}}\\
    &\quad+ \frac{t_{+}}{t_{+}-t_{-}} \log \abs{t_{+}} - \frac{t_{-}}{t_{+}-t_{-}} \log \abs{t_{-}}\\
    &=\alpha \Delta_{1-0}(x,u) +\frac{1}{t_{+}-t_{-}} (\alpha-t_{+})\log \abs{\alpha-t_{+}}-\frac{1}{t_{+}-t_{-}} (\alpha-t_{-})\log \abs{\alpha-t_{-}}\\
    &\quad+ \frac{t_{+}}{t_{+}-t_{-}} \log \abs{t_{+}} - \frac{t_{-}}{t_{+}-t_{-}} \log \abs{t_{-}}.
\end{aligned}$}
\end{equation*}
Here, the last equality is derived by using $t_{+}+t_{-}=1$. By Lemma \hyperref[lemA.4]{5}, the maximizer is $\alpha^{*}_{\text{approx}}(\epsilon, \Delta_{1-0}(x,u))=t_{-}+\sqrt{1+4\epsilon}\sigma(\sqrt{1+4\epsilon}\Delta_{1-0}(x,u))$ where $\sigma(x):=1/(1+e^{-x})$ is the sigmoid function, so $\alpha_{\text{approx}}^{*}(\epsilon, \Delta_{1-0}(x,u))-t_{+}=-(t_{+}-t_{-})(1-\sigma((t_{+}-t_{-})\Delta_{1-0}(x,u)))$ and $\alpha^{*}_{\text{approx}}(\epsilon, \Delta_{1-0}(x,u))-t_{-}=(t_{+}-t_{-})\sigma((t_{+}-t_{-})\Delta_{1-0}(x,u))$. Now, substituting these equations and $(t_{+}-t_{-})\Delta_{1-0}(x,u)=\log \sigma((t_{+}-t_{-})\Delta_{1-0}(x,u))/(1-\sigma((t_{+}-t_{-})\Delta_{1-0}(x,u)))$ gives
\begin{equation*}\resizebox{1.0\hsize}{!}{
$\begin{aligned}
    &\text{LB}(\alpha_{\text{approx}}^{*}(\epsilon, \Delta_{1-0}(x,u)), \epsilon, \Delta_{1-0}(x,u))\\
    &=t_{-} \Delta_{1-0}(x,u) + \sigma((t_{+}-t_{-})\Delta_{1-0}(x,u)) \log \frac{\sigma((t_{+}-t_{-})\Delta_{1-0}(x,u))}{1-\sigma((t_{+}-t_{-})\Delta_{1-0}(x,u))}\\
    &\quad -\sigma((t_{+}-t_{-})\Delta_{1-0}(x,u)) \log \Big( (t_{+}-t_{-})\sigma((t_{+}-t_{-})\Delta_{1-0}(x,u)) \Big)\\
    &\quad-(1-\sigma((t_{+}-t_{-})\Delta_{1-0}(x,u))) \log \Big( (t_{+}-t_{-})(1-\sigma((t_{+}-t_{-})\Delta_{1-0}(x,u))) \Big)\\
    &\quad+ \frac{t_{+}}{t_{+}-t_{-}} \log \abs{t_{+}} - \frac{t_{-}}{t_{+}-t_{-}} \log \abs{t_{-}}\\
    &=t_{-} \Delta_{1-0}(x,u)-\log(1-\sigma((t_{+}-t_{-})\Delta_{1-0}(x,u)))\\
    &\quad +\Big( -\log\abs{t_{+}-t_{-}}+\frac{t_{+}}{t_{+}-t_{-}} \log \abs{t_{+}} - \frac{t_{-}}{t_{+}-t_{-}} \log \abs{t_{-}} \Big)\\
    &=\frac{1+\sqrt{1+4\epsilon}}{2} \Delta_{1-0}(x,u)-\log \sigma (\sqrt{1+4\epsilon}\Delta_{1-0}(x,u)))\\
    &\quad +\Big( -\log(\sqrt{1+4\epsilon})+ \frac{1}{2\sqrt{1+4\epsilon}} \log \frac{(\sqrt{1+4\epsilon}+1)^{2}}{4\epsilon} + \frac{1}{2}\log \epsilon \Big).
\end{aligned}$}
\end{equation*}
Thus, 
\begin{equation*}\resizebox{1.0\hsize}{!}{
$\begin{aligned}
    &\underset{\alpha \in [0, 1]}{\text{sup}}\text{ELBO}_{\theta, \phi}(\alpha; x, u)-\text{ELBO}_{\theta, \phi}(0; x,u)\\
    &\geq \text{LB}(\alpha_{\text{approx}}^{*}(\epsilon, \Delta_{1-0}(x,u)), \epsilon, \Delta_{1-0}(x,u))\\
    & = \frac{1+\sqrt{1+4\epsilon}}{2} \Delta_{1-0}(x,u)-\log \sigma (\sqrt{1+4\epsilon}\Delta_{1-0}(x,u))) + O(\epsilon \log \epsilon)
\end{aligned}$}
\end{equation*}
as $\epsilon \to 0^{+}$. Now, $-\log \sigma (\sqrt{1+4\epsilon}\Delta_{1-0}(x,u)))=o(\Delta_{1-0}(x,u))$ as $\Delta_{1-0}(x,u) \to \infty$ and $-\log \sigma (\sqrt{1+4\epsilon}\Delta_{1-0}(x,u)))=-\sqrt{1+4\epsilon}\Delta_{1-0}(x,u)+o(\Delta_{1-0}(x,u))$ as $\Delta_{1-0}(x,u) \to -\infty$ conclude the proof.

\section{Details on Experiments}\label{app.B}
\subsection{Implementation Details}\label{app.B.1}
\subsubsection{Dataset Description and Experimental Setting}
\begin{table*}[t]
\caption{A summary of distributions of variables in the simulation study. Examples can be founded at the first column in Figure \hyperref[fig:simulation_results]{1} in the manuscript.}
\label{tab:simulation_description}
\vskip 0.15in
\begin{center}
\begin{small}
\begin{sc}
\resizebox{2.0\columnwidth}{!}{
\begin{tabular}{ccccc}
\toprule
Latent structure &  & \multicolumn{3}{c}{Variables} \\ \cline{3-5}
 &  & Covariates ($U$)      & \begin{tabular}[c]{@{}c@{}}Latent variables given \\ covariates ($Z|U$)\end{tabular}       & \begin{tabular}[c]{@{}c@{}}Observations given \\ latent variables ($X|Z$)\end{tabular} \\ \hline \hline
Sine             &  & $\text{Unif}(0, 2\pi)$       & $N\big( (U, 2\text{sin}U)^{T}, (U/4\pi)I_{2}\big)$  & $N(\text{RealNVP}(Z), I_{100})$ \\
Quadratic        &  & $\text{Unif}(-\pi/2, \pi/2)$ &  $N\big((U, U^{2})^{T}, (2U+\pi)/4\pi I_{2}\big)$  &$N(\text{RealNVP}(Z), I_{100})$ \\
Two circles      &  & $\text{Unif}(-\pi, \pi) \times \text{Cat}_{2}(0.5, 0.5)$  & $N\big((U_{2}\text{cos}U_{1}, U_{2}\text{sin}U_{1})^{T}, (-\vert U_{1} \vert+\pi)/10\pi I_{2}\big)$ &  $N(\text{RealNVP}(Z), I_{100})$        \\ \bottomrule
\end{tabular}}
\end{sc}
\end{small}
\end{center}
\vskip -0.1in
\end{table*}
We present in Table \hyperref[tab:simulation_description]{1} the three data generation schemes used in the simulation study, which include: 1) distributions of covariates ($U$), 2) conditional distributions of latent variables given covariates ($Z|U$), and 3) conditional distributions of observations given latent variables ($X|Z$). Here, uniform and categorical distributions are denoted by Unif and Cat, respectively, and RealNVP \citep{dinh2016density} is a flexible and invertible neural network mapping low-dimensional latent variables to high-dimensional observations. As in \cite{zhou2020learning}, we use randomly initialized RealNVP networks as ground-truth mixing functions. The sample size is $30,000$ and the proportion of training, validation, and test samples are 80\%, 10\%, and 10\%, respectively. The dimension of observations is $100$. The number of repeats is $20$, and for all datasets and methods, we train five models with different initial weights. All reported results are from models yielding the minimum validation loss and evaluated on the test dataset.

We provide descriptions on real datasets with implementation details.

\textbf{EMNIST}: An image dataset consisting of handwritten digits whose data format is the same as MNIST \citep{lecun1998mnist} and has six split types. We use EMNIST split by digits to use images as observations ($X$) and digit labels as covariates ($U$). The official training dataset contains 240,000 images of digits from $0$ to $9$ in $28 \times 28$ gray-scale, and the test dataset contains 40,000 images. We randomly split the official training images by 200,000 and 40,000 images to make training and validation datasets for our experiments, and the number of repeats is $20$.

\textbf{Fashion-MNIST}: An image dataset consisting of fashion-item images with item labels. There are ten classes such as ankle boot, bag, and coat, and we use images as observations ($X$) and item labels as covariates ($U$). The official training and test datasets contain 60,000 and 10,000 images in $28 \times 28$ gray-scale, respectively, and we randomly split the official training images by 50,000 and 10,000 images to make training and validation datasets. The number of repeats is $20$.

\textbf{ABCD}: The ABCD study recruited 11,880 children aged 9–10 years (and their parents/guardians) were across 22 sites with 10-year-follow-up. For this analysis, we are using the baseline measures. After list-wise deletion for missing values, the sample size is 5,053, and the dimension of observations is 1,178. We conduct $5$-fold cross-validation. For all data splits and methods, we train four models with different initial weights. All reported results are from the model yielding the minimum loss on the validation fold and evaluated on the test fold.

\begin{table}[t]
\caption{Contingency tables to display the number of data by their $\alpha^{*}$ computed by grid search (column) and formula (row) on sine latent structure. Correlation coefficients between $\alpha^{*}$ by grid search and by formula are presented at the top-left corner.}
\label{tab:simulation_alpha_table}
\vskip 0.15in
\begin{center}
\begin{small}
\begin{sc}
{    % for making group where "\makegapedcells" is valid
    %\makegapedcells
    \resizebox{1.0\columnwidth}{!}{
    \begin{tabular}{cc|ccc}
    \toprule
    \multicolumn{5}{l}{Latent structure = Sine} \\ \hline
    \multicolumn{2}{c}{} & \multicolumn{3}{c}{Formula} \\
        &   \begin{tabular}[c]{@{}l@{}}Correlation coefficient:\\ 0.99 (0.00)\end{tabular} &   $0$ &  In-between $0$ and $1$ & $1$ \\ 
        \cline{2-5}
    \multirow{2}{*}{\rotatebox[origin=c]{90}{Grid}}
        & $0$  & 26.85\% (0.18\%) & 1.88\% (0.06\%) & 0.00\% (0.00\%) \\
        & In-between $0$ and $1$ & 0.19\% (0.01\%)    & 0.14\% (0.02\%)      & 0.01\% (0.00\%) \\
        & $1$ & 0.00\% (0.00\%)     & 1.99\% (0.05\%)      & 68.95\% (0.19\%) \\ \hline \hline
        
        \end{tabular}
        }
    \resizebox{1.0\columnwidth}{!}{
    \begin{tabular}{cc|ccc}
    \multicolumn{5}{l}{Latent structure = Quadratic} \\ \hline 
    \multicolumn{2}{c}{} & \multicolumn{3}{c}{Formula} \\
        &    \begin{tabular}[c]{@{}l@{}}Correlation coefficient:\\ 0.99 (0.00)\end{tabular}   &   $0$ &  In-between $0$ and $1$ & $1$ \\ 
        \cline{2-5}
    \multirow{2}{*}{\rotatebox[origin=c]{90}{Grid}}
        & $0$  & 30.95\% (0.19\%) & 2.61\% (0.06\%) & 0.00\% (0.00\%) \\
        & In-between $0$ and $1$ & 0.30\% (0.02\%)    & 0.28\% (0.02\%)      & 0.04\% (0.01\%) \\
        & $1$ & 0.00\% (0.00\%)     & 2.68\% (0.08\%)      & 63.14\% (0.26\%) \\ \hline \hline
        
        \end{tabular}
        }
    \resizebox{1.0\columnwidth}{!}{
    \begin{tabular}{cc|ccc}
    \multicolumn{5}{l}{Latent structure = Two circles} \\ \hline
    \multicolumn{2}{c}{} & \multicolumn{3}{c}{Formula} \\
        &    \begin{tabular}[c]{@{}l@{}}Correlation coefficient:\\ 0.99 (0.00) \end{tabular}   &   $0$ &  In-between $0$ and $1$ & $1$ \\ 
        \cline{2-5}
    \multirow{2}{*}{\rotatebox[origin=c]{90}{Grid}}
       & $0$  & 36.05\% (0.19\%) & 3.44\% (0.08\%) & 0.00\% (0.00\%) \\
        & In-between $0$ and $1$ & 0.39\% (0.02\%)    & 0.29\% (0.02\%)      & 0.02\% (0.01\%) \\
        & $1$ & 0.00\% (0.00\%)     & 3.46\% (0.08\%)      & 56.36\% (0.27\%) \\ \hline
        
        \end{tabular}
        }
     }
\end{sc}
\end{small}
\end{center}
\vskip -0.1in
\end{table}

\begin{figure}[t]
\vskip 0.2in
\begin{center}
\centerline{\includegraphics[width=1.0\columnwidth]{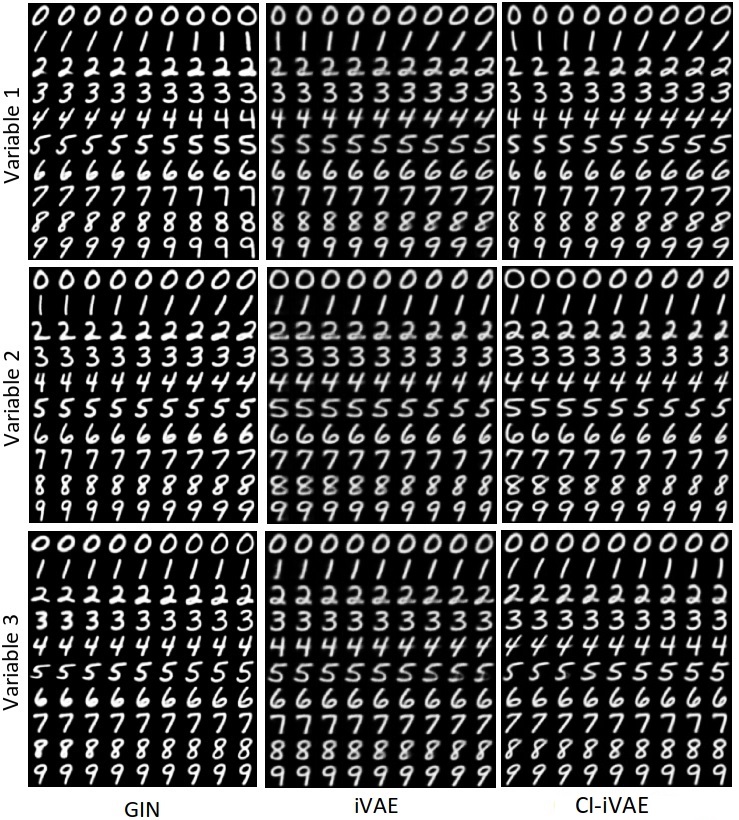}}
\end{center}
\caption{Generation results on EMNIST by varying top three latent attributes having the largest standard deviations. We calculate mean vector of latent variables and controlling the selected attribute from $-2$ to $+2$ standard deviations.}
\label{fig:variation_plot_EMNIST}
\vskip -0.2in
\end{figure}

\begin{figure}[t]
\vskip 0.2in
\begin{center}
\centerline{\includegraphics[width=1.0\columnwidth]{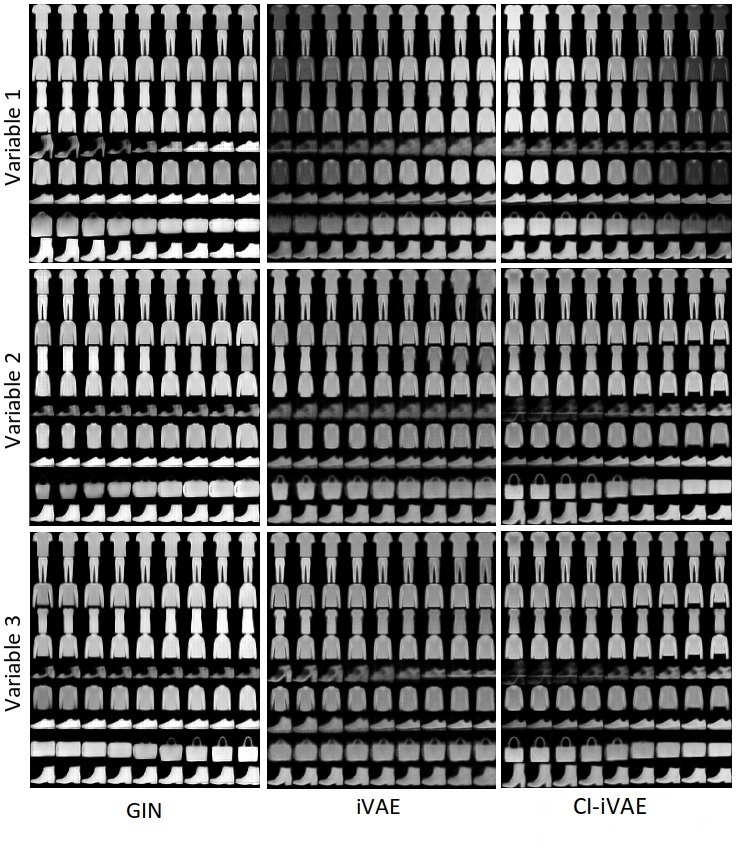}}
\end{center}
\caption{Generation results on Fashion-MNIST by varying top three latent attributes having the largest standard deviations. We calculate mean vector of latent variables and controlling the selected attribute from $-2$ to $+2$ standard deviations.}
\label{fig:variation_plot_Fashion_MNIST}
\vskip -0.2in
\end{figure}

\subsubsection{Network Architectures}
In all experiments, iVAEs and CI-iVAEs use the same architectures of the label prior, encoder, and decoder networks for the purpose of fair comparison.

In the simulation study, we modify the official implementation code of pi-VAE. The architectures of label prior and encoder networks are Dense(60)-Tanh-Dense(60)-Tanh-Dense(2) for the sine latent structure and Dense(60)-Tanh-Dense(60)-Tanh-Dense(60)-Tanh-Dense(2) for quadratic and two circles latent structures. As in pi-VAE, we assume $q(z|x,u) \propto q_{\phi}(z|x)p_{T,\lambda}(z|u)$. The $q(z|x,u)$ is a Gaussian distribution since both label prior and encoder are Gaussian. The means and variances of $q(z|x,u)$ can be computed with those of label prior and encoder. The architecture of the decoder is the same as the modified GIN used in pi-VAE to guarantee injectivity. We use Adam optimizer \citep{kingma2014adam}. The number of epochs, batch size, and the learning rate is $100$, $300$, and $5 \times 10^{-4}$, respectively.

In experiments on EMNIST and Fashion-MNIST datasets, we modify the official implementation code of GIN. For GIN, we use the same architecture used in the GIN paper. For iVAEs and CI-iVAEs, the architecture of encoders is Conv(32, 3, 1, 1)-BN-LReLU-Conv(64, 4, 2, 1)-BN-LReLU-Conv(128, 4, 2, 1)-BN-LReLU-Conv(128, 7, 1, 0)-BN-LReLU-Dense(64), that of decoders is ConvTrans(128, 1, 1, 0)-BN-LReLU-ConvTrans(128, 7, 1, 0)-BN-LReLU-ConvTrans(64, 4, 2, 1)-BN-LReLU-ConvTrans(32, 4, 2, 1)-BN-LReLU-ConvTrans(1, 3, 1, 1)-Sigmoid, and that of the label prior is Dense(256)-LReLU-Dense(256)-LReLU-Dense(64). Here, Conv($f$, $k$, $s$, $p$) and ConvTrans($f$, $k$, $s$, $p$) denote the convolution layer and transposed convolution layer \citep{zeiler2010deconvolutional}, respectively, where $f$, $k$, $s$, and $p$ are the number of output channel, kernel size, stride, and padding, respectively. BN denotes the batch normalization layer \citep{ioffe2015batch}, and LReLU denotes the Leaky ReLU activation layer \citep{xu2015empirical}. The initialized decoders are not injective, but our objective functions encourage them to be injective by enforcing the inverse relation between encoders and decoders. The number of learnable parameters of GIN architectures is 2,620,192, and that of iVAEs and CI-iVAEs is 2,062,209. We use Adam optimizer. The number of epochs and batch size is $100$ and $240$, respectively. The learning rate is $3 \times 10^{-4}$ for the first $50$ epochs and is $3 \times 10^{-5}$ for the remaining epochs.

In the experiment on the ABCD dataset, the architecture of the label prior is Dense(256)-LReLU-Dense(256)-LReLU-Dense(128), that of encoders is Dense(4096)-BN-LReLU-Dense(4096)-BN-LReLU-Dense(4096)-BN-LReLU-Dense(4096)-BN-LReLU-Dense(128)-BN-LReLU-Dense(128), and that of decoders is Dense(4096)-LReLU-Dense(4096)-LReLU-Dense(4096)-LReLU-Dense(4096)-LReLU-Dense(128)-LReLU-Dense(128). We use Adam optimizer. The number of epochs, batch size, and the learning rate is $100$ and $64$, and $2 \times 10^{-4}$, respectively.

\subsection{Further Experimental Results}\label{app.B.2}
We present further experimental results on the simulation study in Table \hyperref[tab:simulation_alpha_table]{2}.

We present contingency tables for samplewise optimal $\alpha$ computed by grid search and by using approximating formula (Equation \eqref{eq:alpha_formula}) in Table \hyperref[tab:simulation_alpha_table]{2}. For grid search, we calculate $\text{ELBO}_{\theta, \phi}(\alpha; x, u)$ for $\alpha \in \{0, 0.001, ..., 0.999, 1\}$ and pick the maximizer. For formula, we approximate $\epsilon$ with $f(\textbf{z})=z_{j}$ and $f(\textbf{z})=z^{2}_{j}$ for $j=1,...,d_{Z}$ and calculate $\alpha_{\text{approx}}^{*}$. For all three settings, the correlation coefficients are high, which indicates the consistency of $\alpha^*$ from the proposed algorithm with theoretical approximation. Moreover, $\alpha^{*}$ does not degenerate at $0$ or $1$, so the proposed ELBO using samplewise optimal posteriors is different from the two ablation cases, ELBOs with $q_{\phi}(z|x,u)$ and with $q_{\phi}(z|x)$.

Generation results according to attributes having the largest standard deviations are provided in Figures \hyperref[fig:variation_plot_EMNIST]{1} and \hyperref[fig:variation_plot_Fashion_MNIST]{2}. For all methods, the generated result changes as the value of attributes are changed. In Fashion-MNIST, iVAE-based methods change the contrast of fashion items while GIN does not.

\bibliography{ref}

\end{document}